\documentclass[a4paper]{article}
\usepackage{graphicx}
\usepackage{listings}
\usepackage{xcolor}
\usepackage{amstext}
\usepackage{custom}
\usepackage[ruled,linesnumbered]{algorithm2e}

\usepackage{geometry}
\usepackage{float}

\usepackage{amsmath,amsfonts,bm}

\def\eqref#1{equation~\ref{#1}}

\def\1{\bm{1}}

\DeclareMathAlphabet{\mathsfit}{\encodingdefault}{\sfdefault}{m}{sl}
\SetMathAlphabet{\mathsfit}{bold}{\encodingdefault}{\sfdefault}{bx}{n}

\newcommand{\E}{\mathbb{E}}

\newcommand{\R}{\mathbb{R}}

\usepackage{bbm}
\usepackage{hyperref}
\usepackage{url}

\usepackage{mathrsfs}
\usepackage{booktabs}
\usepackage{subcaption}

\usepackage{multirow}
\usepackage[all]{xy}
\usepackage{wrapfig}

\newcommand{\bP}{\mathbb{P}}
\newcommand{\bE}{\mathbb{E}}
\renewcommand{\1}{\mathbbm{1}}
\newcommand{\Df}{{\rm Diff}}

\title{FaiREE: Fair Classification with Finite-Sample and Distribution-Free Guarantee}
\author{
Puheng Li\footnotemark[1]\, \footnotemark[4]
\and
James Zou\footnotemark[2]
\and
Linjun Zhang\footnotemark[3]\, \footnotemark[5]
}
\begin{document}
\date{}
\maketitle
\renewcommand{\thefootnote}{\fnsymbol{footnote}} 
\footnotetext[1]{Stanford University. Email: \href{mailto:puhengli@stanford.edu}{puhengli@stanford.edu}.}
\footnotetext[2]{Stanford University. Email: \href{mailto:jamesz@stanford.edu}{jamesz@stanford.edu}.}
\footnotetext[3]{Rutgers University. Email: \href{mailto:linjun.zhang@rutgers.edu}{linjun.zhang@rutgers.edu}.}
\footnotetext[4]{This work was done when Puheng Li was an undergraduate student at Peking University.}
\footnotetext[5]{Corresponding author.}
\renewcommand{\thefootnote}{\arabic{footnote}}

\begin{abstract}
Algorithmic fairness plays an increasingly critical role in machine learning research. Several group fairness notions and algorithms have been proposed. However, the fairness guarantee of existing fair classification methods mainly depends on specific data distributional assumptions, often requiring large sample sizes, and fairness could be violated when there is a modest number of samples, which is often the case in practice. In this paper, we propose FaiREE, a fair classification algorithm that can satisfy group fairness constraints with finite-sample and distribution-free theoretical guarantees. FaiREE can be adapted to satisfy various group fairness notions (e.g., Equality of Opportunity, Equalized Odds, Demographic Parity, etc.) and achieve the optimal accuracy. These theoretical guarantees are further supported by experiments on both synthetic and real data. FaiREE is shown to have favorable performance over state-of-the-art algorithms.
\end{abstract}

\section{Introduction}


As machine learning algorithms have been increasingly used in consequential domains such as college admission \cite{chouldechova2018frontiers}, loan application \cite{ma2018study}, and disease diagnosis \cite{fatima2017survey}, there are emerging concerns about the algorithmic fairness in recent years. When standard machine learning algorithms are directly applied to the biased data provided by humans, the outputs are sometimes found to be biased towards certain sensitive attribute that we want to protect (race, gender, etc). To quantify the fairness in machine learning algorithms, many fairness notions have been proposed, including the individual fairness notion \cite{biega2018equity}, group fairness notions such as Demographic Parity, Equality of Opportunity, Predictive Parity, and Equalized Odds \cite{dieterich2016compas,hardt2016equality,gajane2017formalizing,verma2018fairness}, and multi-group fairness notions including multi-calibration \cite{hebert2018multicalibration} and multi-accuracy \cite{kim2019multiaccuracy}. 
Based on these fairness notions or constraints, corresponding algorithms were designed to help satisfy the fairness constraints \cite{hardt2016equality,pleiss2017fairness,zafar2017fairness,krishnaswamy2021fair,valera2018enhancing, chzhen2019leveraging, zeng2022bayes, thomas2019preventing}.
\begin{figure}
    \centering
    \includegraphics[width=0.95\textwidth]{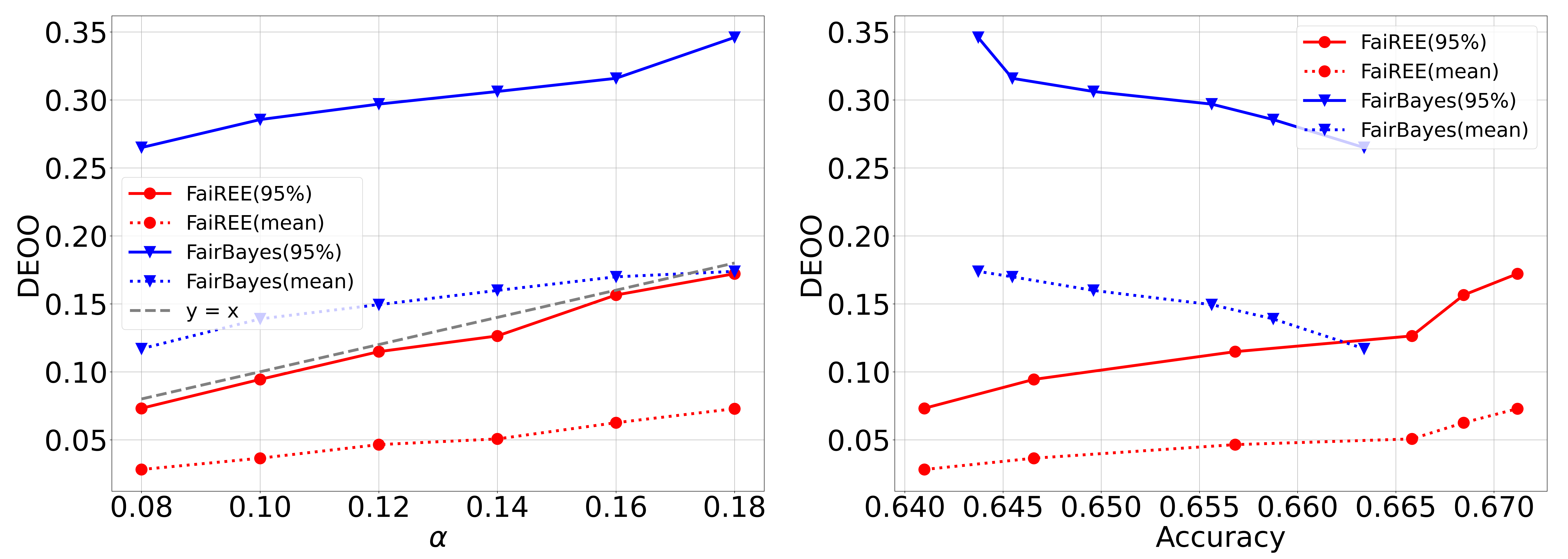}
    \caption{Comparison of FairBayes and FaiREE on the synthetic data with sample size = 1000. See Table \ref{table1} for detailed numerical results. Left: $DEOO$ v.s. $\alpha$, Right: DEOO v.s. Test accuracy. Here, $DEOO$ is the degree of violation to fairness constraint Equality of Opportunity and $\alpha$ is the pre-specified desired level to upper bound $DEOO$ for both methods. See Eq.~(\ref{deoo}) in Section~\ref{section2} for a more detailed definition. 
    }
    \label{compare}
\end{figure}

Among these fairness algorithms, post-processing  is a popular type of algorithm which modifies the output of the model to satisfy fairness constraints. However, recent post-processing algorithms are found to lack the ability to realize accuracy–fairness trade-off and perform poorly when the sample size is limited \cite{hardt2016equality,pleiss2017fairness}. 
In addition, since most fairness constraints are non-convex, some papers propose convex relaxation-based methods \cite{zafar2017fairness,krishnaswamy2021fair}. This type of algorithms generally do not have the theoretical guarantee of how the output satisfies the exact original fairness constraint.  
Another line of research considers recalibrating the Bayes classifier by
a group-dependent threshold \cite{valera2018enhancing, chzhen2019leveraging, zeng2022bayes}. However, their results require either some distributional assumptions or infinite sample size, which is hard to verify/satisfy in practice.

In this paper, we propose a post-processing algorithm FaiREE that provably achieves group fairness {guarantees} with only finite-sample and {free of distributional assumptions} (this property is also called ``distribution-free'' in the literature \cite{maritz1995distribution, clarke2007simple, gyorfi2002distribution}).  To the best of our knowledge, this is the first algorithm in \textcolor{black}{fair} classification with a finite-sample and distribution-free guarantee. A brief pipeline of FaiREE is to first score the dataset with the given classifier, and select a candidate set based on these scores which can fit the fairness constraint with a theoretical guarantee. As there are possibly multiple classifiers that can satisfy this constraint, we further develop a distribution-free estimate of the test mis-classification error, resulting in an algorithm that produces the optimal mis-classification error given the fairness constraints. As a motivating example, Figure \ref{compare} shows that applying state-of-the-art FairBayes method in \cite{zeng2022bayes} on a dataset with 1000 samples results in substantial fairness violation on the test data and incorrect behavior of fairness-accuracy trade-off due to lack of fairness generalization. Our proposed FaiREE improved fairness generalization in these finite sample settings. 

\paragraph{Additional Related Works.} The fairness algorithms in the literature can be roughly categorized into three types: 1). Pre-processing algorithms that learn a fair representation to improve fairness \cite{zemel2013learning,louizos2015variational,lum2016statistical,adler2018auditing,calmon2017optimized,gordaliza2019obtaining,madras2018learning, kilbertus2020fair} 2). In-processing algorithms that optimize during training time \cite{calders2009building,woodworth2017learning,zafar2017fairness,zafar2017fairness1,agarwal2018reductions,russell2017worlds, zhang2018mitigating, celis2019classification} 3). Post-processing algorithms that try to modify the output of the original method to fit fairness constraints \cite{kamiran2012decision,feldman2015computational,hardt2016equality,fish2016confidence,pleiss2017fairness,corbett2017algorithmic,menon2018cost,hebert2018multicalibration,kim2019multiaccuracy,deng2023happymap}.

The design of post-processing algorithms with distribution-free and finite-sample guarantees gains much attention recently due to its flexibility in practice \cite{shafer2008tutorial, romano2019conformalized}, as it can be applied to any given algorithm (eg. a black-box neural network), and achieve desired theoretical guarantee with almost no assumption. One of the research areas that satisfies this property is conformal prediction \cite{shafer2008tutorial,lei2018distribution,romano2019conformalized} whose aim is to construct prediction intervals that cover a future response with high probability. In this paper, we extend this line of research beyond prediction intervals, by designing classification algorithms that satisfy certain group fairness with distribution-free and finite-sample guarantees. 

\textbf{Paper Organization.} Section \ref{section2} provides the definitions and notations we use in the paper. Section \ref{s3} provides the general pipeline of FaiREE. In Section \ref{s4}, we further extend the results to other fairness notions. Finally, Section \ref{section5} conducts experiments on both synthetic and real data and compares with several state-of-art algorithms to show that FaiREE has desirable performance \footnote{Code is available at \url{https://github.com/lphLeo/FaiREE}}.

\section{Preliminary}\label{section2}
In this paper, we consider two types of features in classification: the standard feature $X \in \mathcal{X}$, and the sensitive attribute, which we want the output to be fair on, is denoted as $A \in \mathcal{A} = \{0, 1\}$. For the simplicity of presentation, we consider the binary classification problem with labels in $\mathcal{Y} = \{0, 1\}$. We note that our analysis can be similarly extended to the multi-class and multi-attribute setting. Under the binary classification setting, we use the score-based classifier that outputs a prediction $\widehat{Y}=\widehat{Y}(x,a) \in \{0, 1\}$ based on a score function $f(x,a)\in[0,1]$ that depends on $X$ and $A$:
    \begin{definition}(Score-based classifier)
        A score-based classifier is an indication function $\hat Y=\phi(x,a) = \1\{f(x,a) > c\}$ for a measurable score function $f: \mathcal{X} \times\{0,1\} \rightarrow[0,1]$ and some threshold $c>0$. 
        

    \end{definition}

    To address the algorithmic fairness problem, several group fairness notions have been developed in the literature. In the following, we introduce two of the popular notions, Equality of Opportunity and Equalized Odds. We will discuss other fairness notions in Section~\ref{A7} of the Appendix.
    
    Equality of Opportunity requires comparable true positive rates across different protected groups. 
\begin{definition}(Equality of Opportunity \cite{hardt2016equality})
A classifier satisfies Equality of Opportunity if it satisfies the same true positive rate among protected groups:
$\bP_{X \mid A=1, Y=1}(\widehat{Y}=1)=\bP_{X \mid A=0, Y=1}(\widehat{Y}=1).$

\label{def-eq}
\end{definition}

Equalized Odds is an extension of Equality of Opportunity, requiring both false positive rate and true positive rate are similar across different attributes.

    \begin{definition}(Equalized Odds \cite{hardt2016equality})
        A classifier satisfies Equalized Odds if it satisfies the following equality:
           $\bP_{X \mid A=1, Y=1}(\widehat{Y}=1) = \bP_{X \mid A=0, Y=1}(\widehat{Y}=1)$ and 
        $\bP_{X \mid A=1, Y=0}(\widehat{Y}=0) = \bP_{X \mid A=0, Y=0}(\widehat{Y}=0)$. 
    \end{definition}

Sometimes it is too strict to require the classifier to satisfy Equality of Opportunity or Equalized Odds exactly, which may sacrifice a lot of accuracy (as a very simple example is $f(x, a) \equiv 1$). In practice, to strike a balance between fairness and accuracy, it makes sense to relax the equality above to an inequality with a small error bound.
We use the difference with respect to Equality of Opportunity, denoted by $DEOO$, to measure the disparate impact:
\begin{equation}
DEOO = \bP_{X \mid A=1, Y=1}(\widehat{Y}=1)-\bP_{X \mid A=0, Y=1}(\widehat{Y}=1). 
\label{deoo}
\end{equation}

For a classifier $\phi$, following \cite{zeng2022bayes, cho2020fair}, $|DEOO(\phi)| \leq \alpha$ denotes an $\alpha$-tolerance fairness constraint that controls the difference between the true positive rates below $\alpha$. 

Similarly, we define the following difference with Equalized Odds. Since Equalized Odds, the difference is a two-dimensional vector: 
$$DEO=(\bP_{X \mid A=1, Y=1}(\widehat{Y}=1) - \bP_{X \mid A=0, Y=1}(\widehat{Y}=1), \bP_{X \mid A=1, Y=0}(\widehat{Y}=1) - \bP_{X \mid A=0, Y=0}(\widehat{Y}=1)).$$
   
  For notational simplicity, we use the notation $\preceq$ for the element-wise comparison between vectors, that is, $DEO \preceq (\alpha_1, \alpha_2)$ if and only if $\bP_{X \mid A=1, Y=1}(\widehat{Y}=1) - \bP_{X \mid A=0, Y=1}(\widehat{Y}=1) \leq \alpha_1$ and $\bP_{X \mid A=1, Y=0}(\widehat{Y}=1) - \bP_{X \mid A=0, Y=0}(\widehat{Y}=1) \leq \alpha_2$.
  

\textbf{Additional Notation.} We denote the proportion of group $a$ by  $p_{a}:=\bP(A=a)$ for $a\in \{0,1\}$; the proportion of group $Y = 1$ conditioned on $A$ for $p_{Y, a}:=\bP(Y=1 \mid A=a)$; the proportion of group $Y=1$ conditioned on $A$ and $X$ for $\eta_{a}(x):=\bP(Y=1 \mid A=a, X=x)$. Also, we denote by $\bP_{X}(x)$ and $\bP_{X \mid A=a, Y=y}(x)$ respectively the distribution function of $X$ and the distribution function of $X$ conditioned on $A$ and $Y$. The standard Bayes-optimal classifier without fairness constraint is defined as $\phi^*(x,a)=\1\{f^*(x,a)>1/2\}$, where $f^{*} \in {\arg\min}_{f} [\bP(Y \neq \1\{f(x,a)>1/2\})]$. We denote $v_{(k)}$ as the $k^{th}$ ordered value of sequence $v$ in non-decreasing order. For a set $T$, we denote $sort(T)$ as a function that returns $T$ in non-decreasing order. For a number $a\in\R$, we use $\lceil a \rceil$  to  denote the ceiling function that  maps $a$ to the least integer greater than or equal to $a$. For a positive integer $n$, we use $[n]$ to denote the set $\{1,2,...,n\}$. 

\section{FaiREE: A Finite Sample Based Algorithm}
\label{s3}
In this section, we propose FaiREE, a general post-processing algorithm that produces a Fair classifier in a finite-sample and distribution-fREE manner, and can be applied to a wide range of group fairness notions. We will illustrate its use in Equality of Opportunity as an example in this section, and discuss more applications in later sections.

\subsection{The general pipeline of FaiREE}


Suppose we have dataset $S = S^{0,0}\cup S^{0,1} \cup S^{1,0} \cup S^{1,1}$, where $S^{y,a} = \{x^{y,a}_{1}, \ldots, x^{y,a}_{n^{y,a}}\}$ is the set of features associated with label $Y=y \in \{0,1\}$ and protected attribute $A=a \in \{0,1\}$. We denote the size of $S^{y,a}$ by $n^{y,a}$. Throughout the paper, we assume that $x_i^{y,a}, i \in \{1,\ldots,n^{y,a}\}$ are independently and identically distributed given $Y=y, A=a$. We define $n = n^{0,0} + n^{0,1} + n^{1,0} + n^{1,1}$ to be the total number of samples. Our goal is to post-process any given classifier to make it satisfy certain group fairness constraints.



FaiREE is a post-processing algorithm that can transform any pre-trained classification function score $f$ in order to satisfy fairness constraints. In particular, FaiREE  consists of three main steps, \textit{scoring}, \textit{candidate set construction}, and \textit{candidate selection}. See Figure \ref{fancy} for an illustration. We would like to note that the procedure that first chooses a candidate set of tuning parameters and then selects the best one has been commonly used in machine learning, such as in Seldonian algorithm framework to control safety and fairness \cite{thomas2019preventing,giguere2022fairness,weber2022enforcing}, the Learn then Test framework for risk control \cite{angelopoulos2021learn}, and in high-dimensional statistics \cite{wang2022finite}.  

\begin{figure}[H]
    \centering
    \includegraphics[width=14cm]{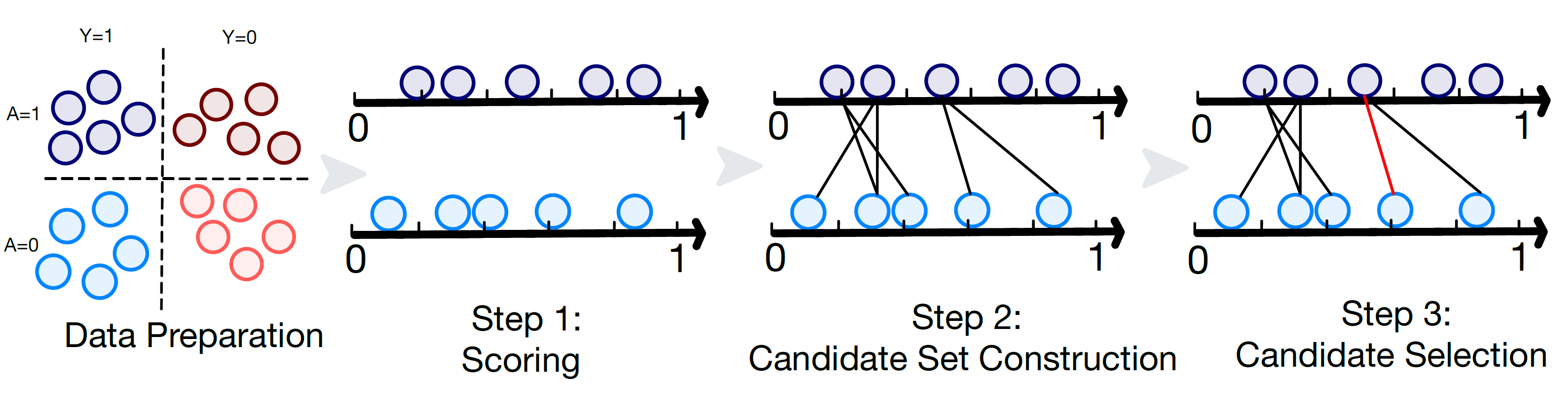}
    \caption{A concrete pipeline of FaiREE for Equality of Opportunity. Edges in Step 2 represent the selected candidate pair and the red edge in Step 3 represents the final optimal candidate selected from all the edges. Each pair represents two different thresholds of a single classifier.} 
    \label{fancy}
\end{figure}

\noindent\textbf{Step 1: Scoring.} FaiREE takes input as
    1). a given fairness guarantee $\mathcal{G}$, such as Equality of Opportunity or Equalized Odds;
    2). an error bound $\alpha$, which controls the violation with respect to our given fairness notion;
    3). a small tolerance level $\delta$, which makes sure our final classifier satisfies our requirement with  probability at least $1-\delta$;
    4). a dataset $S$. 

For \textit{scoring}, we first apply the given classifier $f$ to $S^{y,a}$ and denote the outcome $t_i^{y,a}:=f(x_i^{y,a})$  
as scores for each sample. These scores are then sorted within each subset in non-decreasing order respectively and obtain $T^{y,a} = \{t^{y,a}_{(1)},\ldots,t^{y,a}_{(n^{y,a})}\}$. 

\noindent\textbf{Step 2: Candidate Set Construction.} We first present a key observation for this step, which holds for many group fairness notions such as Equality of Opportunity, Equalized Odds (see details of more fairness notions in Section~\ref{S32}):

{\textit{Any classifier can fit the fairness constraint with high probability by setting the decision threshold appropriately, 
regardless of the data distribution.}}

The insight of this observation comes from recent literature on post-processing algorithms and Neyman-Pearson classification algorithm \cite{fish2016confidence,corbett2017algorithmic, valera2018enhancing,menon2018cost,tong2018neyman,chzhen2019leveraging}. Under Equality of Opportunity, this observation is formalized in Proposition \ref{p1}. We also establish similar results under other fairness notions beyond Equality of Opportunity in Section~\ref{S32}. 
From this observation we can build an algorithm to calculate the probability that a classifier $f$ with a certain threshold will satisfy the fairness constraint: $\Df_\mathcal{G}(f) \leq \alpha$ 
, where $\Df_\mathcal{G}(f)$ is a generic notation to denote the violation rate of $f$ under some fairness notion $\mathcal{G}$. Then we choose the classifiers with the probability $\bP(\Df_\mathcal{G}(f) > \alpha)\leq \delta$ as our candidate set $C$. This candidate set consists of a set of threshold values, with potentially different thresholds for different subpopulations. 

\noindent\textbf{Step 3: Candidate Selection.} Furthermore, as there might be multiple classifiers that satisfy the given fairness constraints, we aim to choose a classifier with a small mis-classification error. To do this, FaiREE estimates the mis-classification error $err(f)$ of the classifier $f$, and chooses the one with the smallest error among the candidate set constructed in the second step.


In the rest part of the section, as an example, we consider Equality of Opportunity as our target group fairness constraint and provide our algorithm in detail. 

\subsection{Application to Equality of Opportunity}
\label{S32}
In this section, we apply FaiREE to the fairness notion Equality of Opportunity. The following two subsections explain the steps \textbf{Candidate Set Construction} and \textbf{Candidate Selection} in detail.
\subsubsection{Candidate Set Construction}
 We first formalize our observation in the following proposition. 
Using the property of order statistics, 
 the following proposition states that it is sufficient to choose the threshold of the score-based classifier from the sorted scores to control the fairness violation in a distribution-free and finite-sample manner. Here, $k^{1,a}$ is the index from which we select the threshold in $T^{1,a}$.
\begin{proposition}\label{p1}
Consider $k^{1,a} \in \{1,\ldots,n^{1,a}\}$ for $a\in\{0,1\}$, and the score-based classifier  $\phi(x,a)=\1\{f(x,a))>t^{1,a}_{(k^{1,a})}\}$. Let $g_{1}(k, a)=\bE[\sum\limits^{n^{1,a}}_{j=k}{n^{1,a}\choose j}(Q^{1,1-a}-\alpha)^{j}(1-(Q^{1,1-a}-\alpha))^{n^{1,a}-j}]$ with $Q^{1,a}\sim Beta(k,n^{1,a}-k+1)$, then we have:
\begin{align*}
    \bP(|DEOO(\phi)|>\alpha) \leq g_{1}(k^{1,1},1)+g_{1}(k^{1,0},0).
\end{align*} 
Additionally, if $t^{1,a}_{(k^{1,a})}$ is a continuous random variable, the inequality above becomes tight equality. 
\end{proposition}
Here, $g_1$ is a function constructed using the property of order statistics so that $g_{1}(k^{1,1},1)$ and $g_{1}(k^{1,0},0)$ upper bound $\bP(DEOO(\phi)>\alpha) $ and $\bP(DEOO(\phi)<-\alpha)$ respectively. We note that $g_1$ can be efficiently compute using Monte Carlo simulations. In our experiments, we approximate $g_1$ by randomly sampling from the Beta distribution for 1000 times and achieve satisfactory approximation. 
This proposition ensures that the $DEOO$ of a given classifier can be controlled with high probability if we choose an appropriate threshold value when post-processing.

Based on the above proposition, we then build our classifiers for an arbitrarily given score function $f$ as below.
We define $L(k^{1,0},k^{1,1})=g_{1}(k^{1,1},1) +g_{1}(k^{1,0},0)$. Recall that the error tolerance is $\alpha$, and $\delta$ is the tolerance level. 
Our candidate set is then constructed as
$K =  \{(k^{1,0}, k^{1,1}) \mid L(k^{1,0}, k^{1,1}) \leq \delta\}.$

Before we proceed to the theoretical guarantee for this candidate set, we introduce a bit more notation. Let us denote the size of the candidate set $K$ by $M$, and the elements in the set $K$ by $(k^{1,0}_1, k^{1,1}_1), \ldots, (k^{1,0}_M,k^{1,1}_M)$.  Additionally, we let $\hat\phi_i(x,a)=\1\{f(x,a)>t^{1,a}_{(k^{1,a}_i)}\}$, for $i=1,\ldots,M$.

To ensure that there exists at least one valid classifier (i.e. $M\geq 1$), we should have $\E[(Q^{1,0}-\alpha)^{n^{1,0}}]+\E[(Q^{1,1}-\alpha)^{n^{1,1}}] \leq \delta$, which requires a necessary and sufficient lower bound requirement on the sample size, as formulated in the following proposition:

\begin{theorem}
If $\min\{n^{1,0}, n^{1,1}\} \geq \lceil \frac{\log \frac{\delta}{2}} {\log (1-\alpha)}\rceil$, for each $i\in\{1,\ldots,M\}$ in the candidate set, we have $|DEOO(\hat\phi_i)| < \alpha$ with probability $1-\delta$.  
\label{p3}

\end{theorem}


As there are at most $n^{1,0}n^{1,1}$ elements in the candidate set, the size of $K$, $M$, can be as large as $O(n^2)$. To further reduce the computational complexity, in the following part, we provide a method to shrink the candidate set.

Our construction is inspired by the following lemma, which gives the analytical form of the  fair Bayes-optimal classifier under the Equality of Opportunity constraint. This Bayes-optimal classifier is defined as $\phi_{\alpha}^* = {\arg\min}_{\mid DEOO(\phi)\mid \leq \alpha} \bP(\phi(x,a) \neq Y)$.

\begin{lemma}[Adapted from Theorem E.4 in \cite{zeng2022bayes}]
The fair Bayes-optimal classifier under Equality of Opportunity can be explicitly written as $\phi_{\alpha}^*(x,a)=\1\{f^*(x,a)>t^*_{a}\}$, then $t^*_1 = \frac{p_{1}p_{Y,1}}{2p_{1}p_{Y,1}-{(1/{t^{1,0}_{(k)}}-2)\cdot p_{0}p_{Y,0}}}$.
\label{pz}
\end{lemma}

 

Note that in practice, the input classifier $f$ can be the classifier trained by a classification algorithm on the training set, which means it is close to $f^*$. Thus from this observation, we can adopt a new way of building a much smaller candidate set. Note that our original candidate set is defined as : $K =  \{(k^{1,0}, k^{1,1}) \mid L(k^{1,0}, k^{1,1}) \leq \delta\} = \{(k^{1,0}_1, k^{1,1}_1), \ldots, (k^{1,0}_M,k^{1,1}_M)\}.$
Now, for every $1 \leq k \leq n^{1,0}$, from Lemma \ref{pz} we denote $u_1(k) = \mathop{\arg\min}\limits_{u} |t^{1,1}_{(u)} - \frac{\hat{p}_1\hat{p}_{Y,1}}{2\hat{p}_1\hat{p}_{Y,1}-{(1/{t^{1,0}_{(k)}}-2)\cdot \hat{p}_0\hat{p}_{Y,0}}}|$, where $\hat{p}_a = \frac{n^{1,a} + n^{0,a}}{n^{0,0} + n^{0,1} + n^{1,0} + n^{1,1}}$ and $\hat{p}_{y,a} = \frac{n^{1,a}}{n^{0,a} + n^{1,a}}$. 
We then build our candidate set as below:
\begin{align}\label{eq:Kprime}
K^{\prime} =  &\{(k^{1,0}, u_1(k^{1,0})) \mid L(k^{1,0}, u_1(k^{1,0})) \leq \delta\}. 
\end{align}


This candidate set $K'$ has cardinality at most $n$. Since our next step, Candidate Selection, has computational complexity that is linear in the size of the candidate set, using the new set $K'$ would help us reduce the computational complexity from $O(n^2)$ to $O(n)$. 

\subsubsection{Candidate Selection}
In this subsection, we explain in detail how we choose the classifier with the smallest mis-classification error from the candidate set constructed in the last step. For a given pair $({k_{i}^{1,0}}, {k_{i}^{1,1}})$ {in the candidate set of index} ($i\in[M]$), we need to know the rank of $t^{1,0}_{(k_{i}^{1,0})}$ and $t^{1,1}_{(k_{i}^{1,1})}$ in the sorted set $T^{0,0}$ and $T^{0, 1}$ respectively in order to compute the test error where we need to consider both $y=0$ and $1$.  Specifically, we find the $k_{i}^{0,a}$ such that $t^{0,a}_{(k_{i}^{0,a})} \leq t^{1,a}_{(k_{i}^{1,a})} < t^{0,a}_{(k_{i}^{0,a} + 1)}$ for $a\in\{0,1\}$.

 To estimate the test mis-classification error of $\hat\phi_i(x,a) = \1\{f(x,a) > t^{1,a}_{(k^{1,a}_{i})}\}$,  
 we divide the error into four terms by different values of $y$ and $a$. We then estimate each part using the property of order statistics respectively, and obtain the following proposition:


\begin{proposition}\label{p4}
Suppose the density functions of $f$ under $A=a, Y=1$ are continuous. Let $\hat{e_i}= \frac{k^{1,0}}{n^{1,0}+1} \frac{n^{1,0}}{n} - \frac{k^{1,1}}{n^{1,1}+1} \frac{n^{1,1}}{n} - (1- \frac{k^{0,0}}{n^{0,0}})\frac{n^{0,0}}{n} - ( 1 - \frac{k^{0,1}}{n^{0,1}})\frac{n^{0,1}}{n}$, for $i=1,2,...,M$. Then, there exist two constants $c_1,c_2 >0$ such that $\mid \bP(\hat{\phi_i}(x,a) \neq Y) -\hat{e_i}\mid \leq  c_1/\sqrt{\min(n^{0,0}, n^{0,1}, n^{1,0}, n^{1,1})}$ with probability larger than $1-c_2\exp(-\min(n^{0,0}, n^{0,1}, n^{1,0}, n^{1,1}))$.
\end{proposition}

The above proposition enables us to efficiently estimate the test error of the $M$ classifiers $\hat\phi_i$'s defined above, from which we can choose a classifier with the lowest test error $\hat\phi$. The algorithm is summarized in Algorithm~\ref{alg:EO}.
\begin{algorithm}[!htb]
\caption{FaiREE for Equality of Opportunity}\label{alg:EO}
\KwIn{Data $S$ = $S^{0,0}\cup S^{0,1} \cup S^{1,0} \cup S^{1,1}$;
the error bound  $\alpha$; the tolerance level $\delta$; a given pre-trained classifier $f$
}
    $T^{y,a}=\{f(x^{y,a}_{1}),\ldots,f(x^{y,a}_{n_{y,a}})\}$\\ 
    $\{t^{y,a}_{(1)},\ldots,t^{y,a}_{(n_{y,a})}\}= $sort($T^{y,a}$) \\
    
    Define $g_{1}(k, a)$ as in Proposition~\ref{p1}, and let 
    $L(k^{1,0},k^{1,1}) = g_{1}(k^{1,1}, 1) + g_{1}(k^{1,0},0)$\\


Build candidate set $K'$ as in Eq.~\ref{eq:Kprime}, and write 
$K' =  \{(k^{1,0}_1, k^{1,1}_1), \ldots, (k^{1,0}_{M'},k^{1,1}_{M'})\}$\\ 

Find $k_{i}^{0,0},k_{i}^{0,1}$: $t^{0,0}_{(k_{i}^{0,0})} \leq t^{1,0}_{(k_{i}^{1,0})} < t^{0,0}_{(k_{i}^{0,0} + 1)}$, $t^{0,1}_{(k_{i}^{0,1})} \leq t^{1,1}_{(k_{i}^{1,1})} < t^{0,1}_{(k_{i}^{0,1} + 1)}$\\
$i_*\leftarrow \mathop{\arg\min}\limits_{i\in[M']} \{\hat{e_i}\}$ ($\hat{e_i}$ is defined in Proposition \ref{p4})\\
\KwOut{$\hat\phi(x,a)=\1\{f(x,a)>t^{1,a}_{(k^{1,a}_{i_*})}\}$}
\end{algorithm}
In the following, we provide the theory showing that the output of Algorithm~\ref{alg:EO} is approaching the optimal mis-classification error under Equality of Opportunity. The following theorem states that the final output of FaiREE has both controlled $DEOO$, and achieved almost minimum mis-classification error when the input classifier is properly chosen.

To facilitate the theoretical analysis, we first introduce the following lemma
which implies that the difference between the output of the function can be controlled by the difference between the input. (i.e. the cumulative distribution function won't increase drastically.)

\begin{lemma}
For a distribution $F$ with a continuous density function, suppose $q(x)$ denotes the quantile of $x$ under $F$, then for $x > y$, we have $F_{(-)}(x - y) \leq q(x) - q(y) \leq F_{(+)}(x-y)$, where $F_{(-)}(x)$ and $F_{(+)}(x)$ are two monotonically increasing functions, $F_{(-)}(\epsilon) > 0, F_{(+)}(\epsilon)>0$ for any $\epsilon > 0$ and $\mathop{lim}\limits_{\epsilon \rightarrow 0} F_{(-)}(\epsilon) = \mathop{lim}\limits_{\epsilon \rightarrow 0} F_{(+)}(\epsilon)=0$.
\label{a2}
\end{lemma}

\begin{theorem}\label{p5}
Given any $\alpha^{\prime} < \alpha$. Set $\delta=c_0/M$ for some $c_0>0$, where $M$ is the candidate set size. Suppose $\min\{n^{1,0}, n^{1,1}\} \geq \lceil \frac{\log \frac{\delta}{2}} {\log (1-\alpha)}\rceil$. $\hat{\phi}$ is the output of FaiREE, then:
\begin{enumerate}
    \item[(1).]  $|DEOO(\hat\phi)| < \alpha$ with probability $1-c_0$.
\item[(2).] Suppose the density functions of $f$ and $f^*$ under $A=a, Y=1$ are continuous. For any $\delta^{\prime},\epsilon_0 > 0$, there exist $0<c<1$ and $c_1 >0$ such that when the input classifier $f(x,a)$ satisfies $\| f - f^* \|_\infty \leq \epsilon_0$ and the constructed candidate set is $K^{\prime}$, we have 
$\bP(\hat{\phi}(x,a) \neq Y) - \bP(\phi_{\alpha^{\prime}}^*(x,a) \neq Y)
\leq 2F^*_{(+)}(2\epsilon_0)+\delta^{\prime}$ with probability larger than $1 - c_1c^{\min\{n^{1,0}, n^{1,1}, n^{0,0}, n^{0,1}\}}$.
\end{enumerate}

\end{theorem}

Theorem \ref{p5} ensures that our classifier will approximate fair Bayes-optimal classifier if the input classifier is close to $f^*$. Here, $\alpha^{\prime} < \alpha$ is any positive constant, which we adopt to promise that our candidate set is not empty. 


\textbf{Remark}: We remark that FaiREE requires no assumption on data distribution except for a minimum sample size. It has the advantage over existing literature which generally imposes different assumptions to data distribution. For example, \cite{chzhen2019leveraging} assumes $\eta(x,a)$ must surpass the level $\frac{1}{2}$ on a set of non-zero measure. \cite{valera2018enhancing} assumes that the shifting threshold of the classifier follows the beta distribution. Also, \cite{zeng2022bayes}'s result only holds for population-level, and the finite-sample version is not studied.

\section{Application to More Fairness Notions}
\label{s4}
\subsection{Equalized Odds}

In this section, we apply our algorithm to the fairness notion Equalized Odds, which has two fairness constraints simultaneously. 
To ensure the two constraints, the algorithm of equalized odds should be different from Algorithm~\ref{alg:EO}. We should consider all $S^{y,a}$ instead of just $S^{1,a}$ when estimating the violation to fairness constraint in the step \textbf{Candidate Set Construction}. Thus we add a function $g_0$ that deals with data with protected attribute $A = 1$, to perfect our algorithm together with $g_1$ defined in the last section.

Similar to Proposition~\ref{p1},  the following proposition assures that choosing an appropriate threshold during post-processing enables the high probability control of a given classifier's $DEO$. 


\begin{proposition}\label{pro:deo}
Given $k^{1,0}, k^{1,1}$ satisfying $k^{1,a} \in \{1,\ldots,n^{1,a}\}$ $(a=0,1)$. Define $\phi(x,a)=\1\{f(x,a))>t^{1,a}_{(k^{1,a})}\}$,
$g_{y}(k,a)=\bE[\sum\limits^{n^{y,a}}_{j=k}{n^{y,a}\choose j}(Q^{y,1-a}-\alpha)^{j}(1-(Q^{y,1-a}-\alpha))^{n^{y,a}-j}]$ with $Q^{y,a}\sim Beta(k+1-y, n^{y,a}-k+y)$,
then we have: 
\begin{align*}
    \bP(|DEO(\phi)|\preceq(\alpha, \alpha)) \geq 1 - g_{1}(k^{1,1},1)-g_{1}(k^{1,0},0) - g_{0}(k^{0,1},1) - g_{0}(k^{0,0},0). 
\end{align*} 

\end{proposition}

Similar to  Proposition \ref{p1}, $g_0$ and $g_1$ jointly control the probability of $\phi$ violating the $DEO$ constraint.






\begin{algorithm}[!htb]
\caption{FaiREE for Equalized Odds}
\KwIn{Training data $S =S^{0,0}\cup S^{0,1} \cup S^{1,0} \cup S^{1,1}$; the error bound  $\alpha$; the tolerance level $\delta$; a given pre-trained classifier $f$
}
     $T^{y,a}=\{f(x^{y,a}_{1}),\ldots,f(x^{y,a}_{n_{y,a}})\}$\\
    $\{t^{y,a}_{(1)},\ldots,t^{y,a}_{(n_{y,a})}\}= $sort($T^{y,a}$) \\
    Define $g_{0}(k,a)$ and $g_{1}(k,a)$ as in Proposition~\ref{pro:deo}, $L_1(k^{1,0},k^{1,1}) = g_{1}(k^{1,1},1) +g_{1}(k^{1,0},0)$, and $L_0(k^{0,0},k^{0,1}) = g_{0}(k^{0,1},1) + g_{0}(k^{0,0},0)$. 
    
For every $k^{1,0}, k^{1,1}$, there exists $k^{0,0}, k^{0,1}$ such that $t^{0,0}_{(k^{0,0})} \leq t^{1,0}_{(k^{1,0})} < t^{0,0}_{(k^{0,0} + 1)}$, $t^{0,1}_{(k^{0,1})} \leq t^{1,1}_{(k^{1,1})} < t^{0,1}_{(k^{0,1} + 1)}$.\\
Build the candidate set as
$K =  \{(k^{1,0}, k^{1,1}) \mid L_1(k^{1,0}, k^{1,1}) + L_0(k^{0,0}, k^{0,1}) \leq \delta\} = \{(k^{1,0}_1, k^{1,1}_1), \ldots, (k^{1,0}_M,k^{1,1}_M)\}$.\\ 

Compute $\hat{e_i}$ as in Proposition \ref{p4}, and let 
$i_*= \mathop{\arg\min}\limits_{i \in [M]} \{\hat{e_i}\}$. \\
\KwOut{$\hat\phi(x,a)=\1\{f(x,a)>t^{1,a}_{(k^{1,a}_{i_*})}\}$}
\end{algorithm}

 Proposition \ref{pro:deo} directly yields the following proposition on the $DEO$ of classifiers in the candidate set.

\begin{theorem}\label{p8}
If $\min\{n^{0,0}, n^{0,1}, n^{1,0}, n^{1,1}\} \geq \lceil \frac{\log \frac{\delta}{4}} {\log (1-\alpha)}\rceil$, then for each $i\in\{1,\ldots,M\}$, we have $|DEO(\hat\phi_i)| \preceq (\alpha,\alpha)$ with probability $1-\delta$.
\end{theorem}

The theoretical analysis of test error is similar to the algorithm for Equality of Opportunity.

\begin{theorem}
Given $\alpha^{\prime} < \alpha$. Set $\delta=c_0/M$ for some $c_0>0$, where $M$ is the candidate set size. Suppose $\min\{n^{0,0}, n^{0,1}, n^{1,0}, n^{1,1}\} \geq \lceil \frac{\log \frac{\delta}{4}} {\log (1-\alpha)}\rceil$. $\hat{\phi}$ is the final output of FaiREE, then:\begin{enumerate}
    \item[(1).] $|DEO(\hat\phi)| \preceq (\alpha, \alpha)$ with probability $1-c_0$.

\item[(2).] Suppose the density functions of $f^*$ under $A=a, Y=1$ are continuous. We denote $\phi_{\alpha^{\prime},\alpha^{\prime}}^* = {\arg\min}_{\mid DEO(\phi)\mid \preceq (\alpha^{\prime}, \alpha^{\prime})} \bP(\phi(x,a) \neq Y)$. For any $\delta^{\prime}, \epsilon_0 > 0$, there exist $0<c<1$ and $c_1 > 0$ such that when the input classifier $f$ satisfies $\mid f(x,a) - f^*(x,a)\mid \leq \epsilon_0$, we have
$\bP(\hat{\phi}(x,a) \neq Y) - \bP(\phi_{\alpha^{\prime},\alpha^{\prime}}^*(x,a) \neq Y) \leq 2F^*_{(+)}(2\epsilon_0) + \delta^{\prime}$ with probability larger than $1 - c_1c^{\min\{n^{1,0}, n^{1,1}, n^{0,0}, n^{0,1}\}}$. 
\end{enumerate}

\end{theorem}



\subsection{On comparing different fairness constraints}
 In this subsection, we further extend our algorithms to more fairness notions. The detailed technical results and derivations are deferred to Section~\ref{A7} in the appendix. Specifically, we compare the sample size requirement to make any given score function $f$ to achieve certain fairness constraint. We note that our algorithm is almost assumption-free, except for the $i.i.d.$ assumption and a necessary and sufficient condition of the sample size. Therefore, we make a chart below to recommend different fairness notions used in practice when the sample sizes are limited. We summarize our results in the following table:

\begin{table}[htbp]
\small
\caption{Sample complexity requirements for FaiREE {to achieve different fairness constraints}. We consider the following fairness notions: DP (Demographic Parity), EOO (Equality of Opportunity), EO (Equalized Odds), PE (Predictive Equality), EA (Equalized Accuracy), and $n^{a} = n^{0,a} + n^{1,a}$.
}\label{tab:comp}
\begin{center}
\resizebox{\columnwidth}{!}{\setlength{\tabcolsep}{1.5mm}{\begin{tabular}{c|c|c|c|c}
\toprule
{DP}&{EOO}&{PE}&{EO}&{EA}\\
\midrule
{$n^{a} \geq \lceil \frac{\log \frac{\delta}{2}} {\log (1-\alpha)}\rceil$}&{$n^{1,a} \geq \lceil \frac{\log \frac{\delta}{2}} {\log (1-\alpha)}\rceil$}&{$n^{0,a} \geq \lceil \frac{\log \frac{\delta}{2}} {\log (1-\alpha)}\rceil$}&{$n^{y,a} \geq \lceil \frac{\log \frac{\delta}{4}} {\log (1-\alpha)}\rceil$}&{$n^{y,a}\ge \lceil \frac{\log \frac{\delta}{4}} {\log (\frac{1-y+(2y-1)p_{Y,|y-a|}-\alpha}{y(2p_{Y,a}-1)+1-p_{Y,a})}}\rceil$}\\
\midrule
\end{tabular}}}
\end{center}
\label{table1}
\end{table}

From this table, we find that Demographic Parity requires the least sample size, Equality of Opportunity and Predictive Equality need a lightly larger sample size, and Equalized Odds is the notion that requires the largest sample size among the first four fairness notions. The sample size requirement for Equalized Accuracy is similar to that of Equalized Odds, but does not have a strict dominance.  

\section{Experiments}\label{section5}

In this section, we conduct experiments to test and understand the effectiveness of FaiREE. For both synthetic data and real data analysis, we compare FaiREE with the following representative methods for fair classification: Reject-Option-Classification (ROC) method in \cite{kamiran2012decision}, Eqodds-Postprocessing (Eq) method in \cite{hardt2016equality}, Calibrated Eqodds-Postprocessing (C-Eq) method in \cite{pleiss2017fairness} and FairBayes method in \cite{zeng2022bayes}.The first three baselines are designed to cope with Equalized Odds and the last one is for Equality of Opportunity.

\subsection{Synthetic data}\label{sec:5.1}

To show the  distribution-free and finite-sample guarantee of FaiREE, we generate the synthetic data from mixed  distributions. Real world data are generally heavy-tailed (\cite{resnick1997heavy}). Thus, we consider the following models with various heavy-tailed distributions for generating synthetic data:
\begin{enumerate}
    \item[\textbf{Model 1.}]
    We generate the protected attribute $A$ and label $Y$ with probability $p_1=\bP(A=1)=0.7, p_0=\bP(A=0)=0.3, p_{y, 1}=\bP(Y=y\mid A=1)=0.7$ and $p_{y, 0}=\bP(Y=y\mid A=0)=0.4$ for $y\in\{0,1\}$. The dimension of features is set to 60, and we generate features with $x^{0,0}_{i,j} \stackrel{i.i.d.}{\sim} t(3)$, where $t(k)$ denotes the $t$-distribution with degree of freedom $k$, $x^{0,1}_{i,j} \stackrel{i.i.d.}{\sim} \chi^2_1$, $x^{1,0}_{i,j} \stackrel{i.i.d.}{\sim} \chi^2_3$ and $x^{1,1}_{i,j} \stackrel{i.i.d.}{\sim} N(\mu, 1)$, where $\mu \sim U(0,1)$ and the scale parameter is fixed to be 1, for $j=1,2,...,60$. 
    \item[\textbf{Model 2.}] 
    We generate the protected attribute $A$ and label $Y$ with the probability, location parameter and scale parameter the same as Model 1. The dimension of features is set to 80, and we generate features with $x^{0,0}_{i,j} \stackrel{i.i.d.}{\sim} t(4)$, $x^{0,1}_{i,j} \stackrel{i.i.d.}{\sim} \chi^2_2$, $x^{1,0}_{i,j} \stackrel{i.i.d.}{\sim} \chi^2_4$ and $x^{1,1}_{i,j} \stackrel{i.i.d.}{\sim} Laplace(\mu, 1)$, for $j=1,2,...,80$.
\end{enumerate} 

For each model, we generate 1000 $i.i.d.$ samples, and the experimental results are summarized in Tables~\ref{table1} and ~\ref{table2}. 
\begin{table}[htbp]
\small
\caption{Experimental studies under Model 1. Here $\overline{|DEOO|}$ denotes the sample average of the absolute value of $DEOO$ defined in Eq.~(\ref{deoo}), and $|DEOO|_{95}$ denotes the sample upper 95\% quantile. $\overline{|DPE|}$ and ${|DPE|}_{95}$ are defined similarly for $DPE$ defined in Eq.~(\ref{dpe}). $\overline{ACC}$ is the sample average of accuracy. We use ``/" in the $DPE$ line because FairBayes and FaiREE-EOO are not designed to control $DPE$.} 
\begin{center}
\resizebox{\columnwidth}{!}{\setlength{\tabcolsep}{1.5mm}{\begin{tabular}{c|c|c|c|c|c|c|c|c|c|c|c|c}
\toprule
{} & {Eq} & {C-Eq} & {ROC} & \multicolumn{3}{|c|}{FairBayes} & \multicolumn{3}{|c|}{FaiREE-EOO}& \multicolumn{3}{|c}{FaiREE-EO}\\
\midrule
{$\alpha$} & {/} &{/} &{/} & {0.08} & {0.12}& {0.16}& {0.08}& {0.12}& {0.16}& {0.08}& {0.12}& {0.16}\\
\midrule
{$\overline{|DEOO|}$} & {0.061} & {0.132}& {0.255}& {0.117}& {0.149}& {0.170}& {0.028}& {0.046}& {0.063}& {0.025}& {0.031}& {0.042}\\
{$|DEOO|_{95}$} & {0.146} & {0.307}& {0.500}& {0.265}& {0.297}& {0.316}& {0.073}& {0.115}& {0.157}& {0.079}& {0.108}& {0.133}\\
{$\overline{|DPE|}$} & {0.051} & {0.029}& {0.511}& {/}& {/}& {/}& {/}& {/}& {/}& {0.039}& {0.042}& {0.045}\\
{${|DPE|}_{95}$} & {0.110} & {0.091}& {0.850}& {/}& {/}& {/}& {/}& {/}& {/}& {0.075}& {0.084}& {0.106}\\
\midrule
{$\overline{ACC}$} & {0.472} & {0.606}& {0.637}& {0.663}& {0.656}& {0.646}& {0.621}& {0.657}& {0.669}& {0.552}& {0.562}& {0.615}\\
\midrule
\end{tabular}}}
\end{center}
\label{table1}
\end{table}

\begin{table}[htbp]
\small
\caption{Experimental studies under Model 2, with the same notation as Table \ref{table1}.}
\begin{center}
\resizebox{\columnwidth}{!}{\setlength{\tabcolsep}{1.5mm}{\begin{tabular}{c|c|c|c|c|c|c|c|c|c|c|c|c}
\toprule
{} & {Eq} & {C-Eq} & {ROC} & \multicolumn{3}{|c|}{FairBayes} & \multicolumn{3}{|c|}{FaiREE-EOO}& \multicolumn{3}{|c}{FaiREE-EO}\\
\midrule
{$\alpha$} & {/} &{/} &{/} & {0.08} & {0.12}& {0.16}& {0.08}& {0.12}& {0.16}& {0.08}& {0.12}& {0.16}\\
\midrule
{$\overline{|DEOO|}$} & {0.063} & {0.105}& {0.237}& {0.251}& {0.320}& {0.324}& {0.027}& {0.047}& {0.073}& {0.028}& {0.035}& {0.047}\\
{$|DEOO|_{95}$} & {0.080} & {0.137}& {0.502}& {0.676}& {0.742}& {0.765}& {0.075}& {0.112}& {0.153}& {0.077}& {0.114}& {0.143}\\
{$\overline{|DPE|}$} & {0.043} & {0.107}& {0.209}& {/}& {/}& {/}& {/}& {/}& {/}& {0.041}& {0.044}& {0.056}\\
{${|DPE|}_{95}$} & {0.066} & {0.144}& {0.443}& {/}& {/}& {/}& {/}& {/}& {/}& {0.071}& {0.090}& {0.127}\\
\midrule
{$\overline{ACC}$} & {0.380} & {0.600}& {0.616}& {0.606}& {0.598}& {0.589}& {0.595}& {0.627}& {0.639}& {0.575}& {0.589}& {0.606}\\
\midrule
\end{tabular}}}
\end{center}
\label{table2}
\end{table}

From these two tables, we find that our proposed FaiREE, when applied to different fairness notions Equality of Opportunity and Equality of Opportunity, is able to control the required fairness violation respectively with high probability, while all the other methods cannot. In addition, although satisfying stronger constraints, the mis-classification error FaiREE is comparable to, and sometimes better than the state-of-the-art methods.


\subsection{Real Data analysis}\label{sec:5.2}

In this section, we apply FaiREE to a real data set, Adult Census dataset \cite{dua2017uci}, whose task is to predict whether a person's income is greater than \$50,000. The protected attribute is gender, and the sample size is 45,222, including 32561 training samples and 12661 test samples. To facilitate the numerical study, we randomly split data into training set, calibration set and test set at each repetition and repeat for 500 times.  FaiREE is compared the existing methods described in the last subsection. Again, as shown in Table~\ref{tab:real}, the proposed FaiREE method controls the fairness constraints at the desired level, and achieve small mis-classification error. More implementation details and experiments on other benchmark datasets are presented in \ref{AA9}. 

\begin{table}[htbp]
\small
\caption{Result of different methods on Adult Census dataset}\label{tab:real}
\begin{center}
\resizebox{\columnwidth}{!}{\setlength{\tabcolsep}{1.5mm}{\begin{tabular}{c|c|c|c|c|c|c|c|c|c|c|c|c}
\toprule
{} & {Eq} & {C-Eq} & {ROC} & \multicolumn{3}{|c|}{FairBayes} & \multicolumn{3}{|c|}{FaiREE-EOO} & \multicolumn{3}{|c}{FaiREE-EO}\\
\midrule
{$\alpha$} & {/} &{/} &{/} & {0.07} & {0.1}& {0.14}& {0.07}& {0.1}& {0.14}& {0.07}& {0.1}& {0.14}\\
\midrule
{$\overline{|DEOO|}$} & {0.043} & {0.087}& {0.031}& {0.107}& {0.101}& {0.104}& {0.034}& {0.039}& {0.066}& {0.002}& {0.039}& {0.067}\\
{$|DEOO|_{95}$} & {0.112} & {0.154}& {0.097}& {0.140}& {0.153}& {0.153}& {0.065}& {0.090}& {0.124}& {0.008}& {0.094}& {0.125}\\
{$\overline{|DPE|}$} & {0.027} & {0.048}& {0.044}& {/}& {/}& {/}& {/}& {/}& {/}& {0.030}& {0.066}& {0.074}\\
{${|DPE|}_{95}$} & {0.058} & {0.105}& {0.117}& {/}& {/}& {/}& {/}& {/}& {/}& {0.056}& {0.078}& {0.086}\\
\midrule
{$\overline{ACC}$} & {0.815} & {0.823}& {0.691}& {0.847}& {0.847}& {0.847}& {0.845}& {0.846}& {0.847}& {0.512}& {0.845}& {0.846}\\
\midrule
\end{tabular}}}
\end{center}
\end{table}


\section{Conclusion and discussion}
In this paper, we propose FaiREE, a post-processing algorithm for fair classification with  theoretical guarantees in a finite-sample and distribution-free manner. FaiREE can be applied to a wide range of group fairness notions and is shown to achieve small mis-classification error while satisfying the fairness constraints. Numerical studies on both synthetic and real data show the practical value of FaiREE in achieving a superior fairness-accuracy trade-off than the state-of-the-art methods.
One interesting direction of future work is to extend the FaiREE techniques to multi-group fairness notions such as multi-calibration \cite{hebert2018multicalibration} and multi-accuracy \cite{kim2019multiaccuracy}.

\section*{Acknowledgements}
The research of Linjun Zhang is partially supported by NSF DMS-2015378. The research of
James Zou is partially supported by funding from NSF CAREER and the Sloan Fellowship. 

\bibliography{ref}
\bibliographystyle{plain}

\appendix

\section{Appendix}
\subsection{Proof of Proposition \ref{p1}}






\begin{proof}
The classifier is
$$ \phi=\left\{
       \begin{aligned}
           \1\{f(x,0)>t^{1,0}_{(k^{1,0})}\},a=0\\
           \1\{f(x,1)>t^{1,1}_{(k^{1,1})}\},a=1
       \end{aligned}
   \right.
   $$
we have:\\
$$
\begin{aligned}
|DEOO(\phi)|&=|\bP(\hat{Y}=1 \mid A=0,Y=1)-\bP(\hat{Y}=1 \mid A=1,Y=1)|\\
&=|\bP(f(x,0)>t^{1,0}_{(k^{1,0})} \mid A=0,Y=1) - \bP(f(x,1)>t^{1,1}_{(k^{1,1})}\mid A=1,Y=1)|\\
&=|1-F^{1,0}(t^{1,0}_{(k^{1,0})}) - [1-F^{1,1}(t^{1,1}_{(k^{1,1})})]|\\
&=|F^{1,1}(t^{1,1}_{(k^{1,1})})-F^{1,0}(t^{1,0}_{(k^{1,0})})|
\end{aligned}
$$
Hence, 
$$
\begin{aligned}
\bP(|DEOO(\phi)|>\alpha)&=\bP(|F^{1,1}(t^{1,1}_{(k^{1,1})})-F^{1,0}(t^{1,0}_{(k^{1,0})})|>\alpha)\\
&=\bP(F^{1,1}(t^{1,1}_{(k^{1,1})})-F^{1,0}(t^{1,0}_{(k^{1,0})})>\alpha)+\bP(F^{1,1}(t^{1,1}_{(k^{1,1})})-F^{1,0}(t^{1,0}_{(k^{1,0})})<-\alpha)\\
&\mathop{=}\limits^{\Delta}A+B.
\end{aligned}
$$

We then have 
$$
\begin{aligned}
A&=\bP(F^{1,1}(t^{1,1}_{(k^{1,1})})-F^{1,0}(t^{1,0}_{(k^{1,0})})>\alpha)\\
&=\bP(F^{1,0}(t^{1,0}_{(k^{1,0})})<F^{1,1}(t^{1,1}_{(k^{1,1})})-\alpha)\\
&\leq \E[\bP(t^{1,0}_{(k^{1,0})}<{F^{1,0}}^{-1}(F^{1,1}(t^{1,1}_{(k^{1,1})})-\alpha))\1\{F^{1,1}(t^{1,1}_{(k^{1,1})})-\alpha>0\} \mid t^{1,1}_{(k^{1,1})}]\\
&=\E\{\bP[\text{at least }k^{1,0}\text{ of }t^{1,0}\text{'s are less than }{F^{1,0}}^{-1}(F^{1,1}(t^{1,1}_{(k^{1,1})})-\alpha) ]\1\{F^{1,1}(t^{1,1}_{(k^{1,1})})-\alpha>0\}\mid t^{1,1}_{(k^{1,1})}\}
\end{aligned}
$$

Following this, we obtain
$$
\begin{aligned}
A& \leq \E\{\sum\limits^{n^{1,0}}_{j=k^{1,0}}\bP[\text{exactly j of the }t^{1,0}\text{'s are less than }{F^{1,0}}^{-1}(F^{1,1}(t^{1,1}_{(k^{1,1})})-\alpha)]\1\{F^{1,1}(t^{1,1}_{(k^{1,1})})-\alpha>0\}\mid t^{1,1}_{(k^{1,1})}\}\\
&=\E\{\sum\limits^{n^{1,0}}_{j=k^{1,0}}\left(\begin{aligned}
n&^{1,0} \\
&j
\end{aligned}\right)\bP[t^{1,0}<{F^{1,0}}^{-1}(F^{1,1}(t^{1,1}_{(k^{1,1})})-\alpha)]^{j}(1-\bP[t^{1,0}<{F^{1,0}}^{-1}(F^{1,1}(t^{1,1}_{(k^{1,1})})-\alpha)])^{n^{1,0}-j}\\&\text{\space\space}\1\{F^{1,1}(t^{1,1}_{(k^{1,1})})-\alpha>0\}\mid t^{1,1}_{(k^{1,1})}\}\\
&\leq \E[\sum\limits^{n^{1,0}}_{j=k^{1,0}}\left(\begin{aligned}
n&^{1,0} \\
&j
\end{aligned}\right)(F^{1,1}(t^{1,1}_{(k^{1,1})})-\alpha)^{j}(1-(F^{1,1}(t^{1,1}_{(k^{1,1})})-\alpha))^{n^{1,0}-j}\mid t^{1,1}_{(k^{1,1})}]
\end{aligned}
$$

Similarly, we have
$$
B \leq \E[\sum\limits^{n^{1,1}}_{j=k^{1,1}}\left(\begin{aligned}
n&^{1,1} \\
&j
\end{aligned}\right)(F^{1,0}(t^{1,0}_{(k^{1,0})})-\alpha)^{j}(1-(F^{1,0}(t^{1,0}_{(k^{1,0})})-\alpha))^{n^{1,1}-j} \mid t^{1,0}_{(k^{1,0})}]
$$
Hence, we have
$$
\begin{aligned}
A+B\leq &\E[\sum\limits^{n^{1,0}}_{j=k^{1,0}}\left(\begin{aligned}
n&^{1,0} \\
&j
\end{aligned}\right)(F^{1,1}(t^{1,1}_{(k^{1,1})})-\alpha)^{j}(1-(F^{1,1}(t^{1,1}_{(k^{1,1})})-\alpha))^{n^{1,0}-j} \mid t^{1,1}_{(k^{1,1})}]\\
&+\E[\sum\limits^{n^{1,1}}_{j=k^{1,1}}\left(\begin{aligned}
n&^{1,1} \\
&j
\end{aligned}\right)(F^{1,0}(t^{1,0}_{(k^{1,0})})-\alpha)^{j}(1-(F^{1,0}(t^{1,0}_{(k^{1,0})})-\alpha))^{n^{1,1}-j}\mid t^{1,0}_{(k^{1,0})}]\\
\leq & \bE[\sum\limits^{n^{1,0}}_{j=k^{1,0}}{n^{1,0}\choose j}(Q^{1,1}-\alpha)^{j}(1-(Q^{1,1}-\alpha))^{n^{1,0}-j}] \\
&+ \bE[\sum\limits^{n^{1,1}}_{j=k^{1,1}}{n^{1,1}\choose j}(Q^{1,0}-\alpha)^{j}(1-(Q^{1,0}-\alpha))^{n^{1,1}-j}]
\end{aligned}
$$
The last inequality holds because $F^{1,a}(t^{1,a}_{(k^{1,a})})$ is stochastically dominated by $Beta(k^{1,a}, n^{1,a} - k^{1,a} + 1)$.

If $t^{1,a}$ is continuous random variable, the equality holds.

Now we complete the proof.
\end{proof}

\subsection{Proof of Lemma \ref{pz}}
We first introduce the lemma (theorem E.4 in \cite{zeng2022bayes}):

\begin{lemma}(Fair Bayes-optimal Classifiers under Equality of Opportunity). Let $E^{\star}=\mathrm{DEOO}\left(f^{\star}\right)$. For any $\alpha>0$, all fair Bayes-optimal classifiers $f_{E, \alpha}^{\star}$ under the fairness constraint $|\mathrm{DEOO}(f)| \leq \alpha$ are given as follows:\\
- When $\left|E^{\star}\right| \leq \alpha, f_{E, \alpha}^{\star}=f^{\star}$\\
- When $\left|E^{\star}\right|>\alpha$, suppose $\bP_{X \mid A=1, Y=1}\left(\eta_{1}(X)=\frac{p_{1} p_{Y, 1}}{2\left(p_{1} p_{Y, 1}-t_{E, \alpha}^{\star}\right)}\right)=0$, then for all $x \in \mathcal{X}$ and $a \in \mathcal{A}$,
$$
f_{E, \alpha}^{\star}(x, a)=I\left(\eta_{a}(x)>\frac{p_{a} p_{Y, a}}{2 p_{a} p_{Y, a}+(1-2 a) t_{E, \alpha}^{\star}}\right)
$$
where $t_{E, \alpha}^{\star}$ is defined as
$$
t_{E, \alpha}^{\star}=\sup \left\{t: \bP_{Y \mid A=1, Y=1}\left(\eta_{1}(X)>\frac{p_{1} p_{Y, 1}}{2 p_{1} p_{Y, 1}-t}\right)> \bP_{Y \mid A=0, Y=1}\left(\eta_{0}(X)>\frac{p_{0} p_{Y, 0}}{2 p_{0} p_{Y, 0}+t}\right)+\frac{E^{\star}}{\left|E^{\star}\right|} \alpha\right\}.
$$
\label{eo}
\end{lemma}

Now we come back to prove Proposition \ref{pz}.

\begin{proof}
    From Lemma \ref{eo}, we have 
    \begin{equation}
        t^*_{0} = \frac{p_{0}p_{Y,0}}{2p_{0}p_{Y,0} + t^{*}_{E,\alpha}}
        \label{t*0}
    \end{equation}
    \begin{equation}
        t^*_{1} = \frac{p_{1}p_{Y,1}}{2p_{1}p_{Y,1} - t^{*}_{E,\alpha}}
        \label{t*1}
    \end{equation}
    Combine Eq.~(\ref{t*0}) and (\ref{t*1}) together and we complete the proof.
\end{proof}

\subsection{Proof of Proposition \ref{p4}}
We first provide a lemma for the mis-classification error of the classifier in the candidate set.
\begin{lemma}
For any $\epsilon >0$, with probability at least $1 - 2e^{-2n^{1,0}}- 2e^{-2n^{1,1}}- 2e^{-2n^{0,0}}- 2e^{-2n^{0,1}}$,
 $$
\mid \bP(Y \neq \hat{Y}) - \frac{k^{1,0}}{n^{1,0}+1} p_0p_{Y,0} - \frac{k^{1,1}}{n^{1,1}+1} p_1p_{Y,1} - (1- \frac{k^{0,0}}{n^{0,0}})p_0(1-p_{Y,0}) - ( 1 - \frac{k^{0,1}}{n^{0,1}})p_1 (1-p_{Y,1})\mid \leq \epsilon.
 $$
\label{lemmamis}
\end{lemma}

We also have the following lemma:
\begin{lemma}
 $F^{0,0}(t^{0,0}_{(k^{1,0})})\sim Beta(k^{0,0}, n^{0,0} - k^{0,0} + 1)$, $F^{0,1}(t^{0,1}_{(k^{0,1})}) \sim Beta(k^{0,1}, n^{0,1} - k^{0,1} + 1)$.
 \label{l3}
\end{lemma}
\begin{proof}[Proof of Lemma \ref{l3}]
Since $F^{0,0}, F^{0,1}$ are the continuous cumulative distribution functions of the $t^{0,0}$'s and $t^{0,1}$'s, we have $F^{0,0}(t^{0,0}),  F^{0,1}(t^{0,1})\sim U(0,1)$, thus $F^{0,0}(t^{0,0}_{(k^{0,0})})$ is the ${k^{0,0}}^{th}$ order statistic of $n^{0,0}$ i.i.d samples from $U(0,1)$ and $F^{0,1}(t^{0,1}_{(k^{0,1})})$ is the ${k^{0,1}}^{th}$ order statistic of $n^{0,1}$ i.i.d samples from $U(0,1)$. 

Thus, from the well known fact of the ordered statistics, we have $F^{0,0}(t^{0,0}_{(k^{0,0})})\sim Beta(k^{0,0}, n^{0,0} - k^{0,0} + 1)$ and $F^{0,1}(t^{0,1}_{(k^{0,1})}) \sim Beta(k^{0,1}, n^{0,1} - k^{0,1} + 1)$.

\end{proof}

Now we come back to the proof of Lemma \ref{lemmamis}:
\begin{proof}[Proof of Lemma \ref{lemmamis}]
The classifier is:\\ 
$$ \hat{\phi}=\left\{
       \begin{aligned}
           \1\{f(x,0)>t^{1,0}_{(k^{1,0})}\},A=0\\
           \1\{f(x,1)>t^{1,1}_{(k^{1,1})}\},A=1
       \end{aligned}
   \right.
   $$

For the mis-classification error:
$$
\begin{aligned}
\bP(Y \neq \hat{Y}) &=\bP(Y=1,\hat{Y}=0) + \bP(Y=0,\hat{Y}=1)\\
&=\bP(Y=1,\hat{Y}=0,A=0)+\bP(Y=1,\hat{Y}=0,A=1)\\
&\text{\space\space} + \bP(Y=0,\hat{Y}=1,A=0) + \bP(Y=0,\hat{Y}=1,A=1)]\\
&=\bP(\hat{Y}=0|Y=1,A=0)\bP(Y=1,A=0) + \bP(\hat{Y}=0|Y=1,A=1)\bP(Y=1,A=1)\\
&\text{\space\space}+\bP(\hat{Y}=1|Y=0,A=0)\bP(Y=0,A=0)+\bP(\hat{Y}=1|Y=0,A=1)\bP(Y=0,A=1)\\
&=\bP(f(x,0) \leq t^{1, 0}_{(k^{1,0})} \mid Y=1,A=0) p_0p_{Y,0}\\ 
&\text{\space\space}+ \bP(f(x,1) \leq t^{1,1}_{(k^{1,1})}\mid Y=1,A=1) p_1p_{Y,1}\\ 
&\text{\space\space}+\bP(f(x,0) \geq t^{1, 0}_{(k^{1,0})}\mid Y=0,A=0) p_0(1-p_{Y,0}) \\
&\text{\space\space}+ \bP(f(x,1) \geq t^{1,1}_{(k^{1,1})}\mid Y=0,A=1) p_1(1-p_{Y,1})\\
&= F^{1,0}(t^{1,0}_{(k^{1,0})}) p_0p_{Y,0} + F^{1,1}(t^{1,1}_{(k^{1,1})}) p_1p_{Y,1}\\
&\text{\space\space}+ (1- F^{0,0}(t^{1,0}_{(k^{1,0})}))p_0(1-p_{Y,0}) + ( 1 - F^{0,1}(t^{1,1}_{(k^{1,1})}))p_1 (1-p_{Y,1})\\
\end{aligned}
$$

Since $F^{1,0}(t^{1,0}_{(k^{1,0})}) \sim Beta(k^{1,0}, n^{1,0} - k^{1,0} + 1)$, $F^{1,1}(t^{1,1}_{(k^{1,1})}) \sim Beta(k^{1,1}, n^{1,1} - k^{1,1} + 1)$, we have, for any $\epsilon > 0$,
$$
\begin{aligned}
\bP (\mid F^{1,0}(t^{1,0}_{(k^{1,0})}) - \frac{k^{1,0}}{n^{1,0}+1}\mid > \epsilon) \leq 2 e^{-2 n^{1,0} \epsilon^2},\\
\bP (\mid F^{1,1}(t^{1,1}_{(k^{1,1})}) - \frac{k^{1,1}}{n^{1,1}+1}\mid > \epsilon) \leq 2 e^{-2 n^{1,1} \epsilon^2}
\end{aligned}
$$

We denote the empirical distribution of $F^{0,a}$ as $F^{0,a}_{n}$, then we have $F^{0,a}_{n}(t^{1,a}_{(k^{1,a})}) = \frac{k^{0,a}}{n^{0,a}}$.

From the Dvoretzky-Kiefer-Wolfowitz Inequality, we have for any $\epsilon >0$, 

$$
\bP(\mathop{\sup}\limits_{x} \mid F^{0,a}(x) - F^{0,a}_{n}(x)\mid > \epsilon) \leq 2 e^{-2n^{0,a}\epsilon^2}
$$

Hence, 
$$
\bP( \mid F^{0,a}_{n}(t^{1,a}_{(k^{1,a})}) - \frac{k^{0,a}}{n^{0,a}}\mid > \epsilon) \leq 2 e^{-2n^{0,a}\epsilon^2}
$$

 Combining above, we have with probability at least $1 - 2e^{-2n^{1,0}}- 2e^{-2n^{1,1}}- 2e^{-2n^{0,0}}- 2e^{-2n^{0,1}}$,
 $$
\mid \bP(Y \neq \hat{Y}) - \frac{k^{1,0}}{n^{1,0}+1} p_0p_{Y,0} - \frac{k^{1,1}}{n^{1,1}+1} p_1p_{Y,1} - (1- \frac{k^{0,0}}{n^{0,0}})p_0(1-p_{Y,0}) - ( 1 - \frac{k^{0,1}}{n^{0,1}})p_1 (1-p_{Y,1})\mid \leq \epsilon.
 $$

Now we complete the proof of Lemma \ref{lemmamis}.

\end{proof}

Next, we continue to prove Proposition \ref{p4}.

\begin{lemma}[Hoeffding's inequality]
Let $X_{1}, \ldots, X_{n}$ be independent random variables. Assume that $X_{i} \in\left[m_{i}, M_{i}\right]$ for every $i$. Then, for any $t>0$, we have
$$
\mathbb{P}\left\{\sum_{i=1}^{n}\left(X_{i}-\mathbb{E} X_{i}\right) \geq t\right\} \leq e^{-\frac{2 t^{2}}{\sum_{i=1}^{n}\left(M_{i}-m_{i}\right)^{2}}}
$$
\label{hfd}
\end{lemma}

\begin{proof}[Proof of Proposition \ref{p4}]


We estimate $p_a$ and $p_{Y,a}$ by $\hat{p}_a=\frac{n^{1,a}+n^{0,a}}{n}$ and $\hat{p}_{Y,a}=\frac{n^{Y,a}}{n}$ ($n$ is the number of the total samples). Here, $\frac{n^{1,a}+n^{0,a}}{n} = \frac{\sum\limits_{i=1}^{n} \1\{Z_i^a = 1\}}{n}$ and $\frac{n^{Y,a}}{n} = \frac{\sum\limits_{i=1}^{n} \1\{Z_i^{Y,a} = 1\}}{n}$, where $Z_i^a \sim B(1, p_a)$ and $Z_i^{Y,a} \sim B(1, p_{Y,a})$. From Hoeffding's inequality (Lemma \ref{hfd}), we have:
$$
\bP(\mid \hat{p}_a - p_a \mid \geq \sqrt{\frac{n^{0,a}}{n}}\epsilon) \leq 2e^{-2n^{0,a}\epsilon^2},
$$
$$
\bP(\mid \hat{p}_{Y,a} - p_{Y,a} \mid \geq \sqrt{\frac{n^{0,a}}{n}}\epsilon) \leq 2e^{-2n^{0,a}\epsilon^2}
$$

Thus, with probability $1 - 4e^{-2n^{0,a}\epsilon^2}$, we have:
$$
\left\{
       \begin{aligned}
\mid \hat{p}_a - p_a \mid &\leq \sqrt{\frac{n^{0,a}}{n}}\epsilon\\
\mid \hat{p}_{Y,a} - p_{Y,a} \mid &\leq \sqrt{\frac{n^{0,a}}{n}}\epsilon\\
       \end{aligned}
   \right.\\
$$

Hence, we have with probability at least $1 - 6(e^{-2n^{0,0}\epsilon^2}+e^{-2n^{0,1}\epsilon^2}) - 2(e^{-2n^{1,0}\epsilon^2}+e^{-2n^{1,1}\epsilon^2})$, 
$$
\begin{aligned}
&|\bP(\hat{\phi_i}(x,a) \neq Y) - \hat{\bP}(\hat{\phi_i}(x,a) \neq Y)|\\
\leq&|\frac{k^{1,0}_i}{n^{1,0}+1}p_0p_{Y,0} + \frac{k^{1,1}_i}{n^{1,1}+1}p_1p_{Y,1}
+ \frac{n^{0,0}-k^{0,0}_i}{n^{0,0}}p_0(1-p_{Y,0}) + \frac{n^{0,1}-k^{0,1}_i}{n^{0,1}}p_1(1-p_{Y,1}) \\
&-[\frac{k^{1,0}_i}{n^{1,0}+1}\hat{p}_0\hat{p}_{Y,0} + \frac{k^{1,1}_i}{n^{1,1}+1}\hat{p}_1\hat{p}_{Y,1}
+ \frac{n^{0,0}-k^{0,0}_i}{n^{0,0}}\hat{p}_0(1-\hat{p}_{Y,0}) + \frac{n^{0,1}-k^{0,1}_i}{n^{0,1}}\hat{p}_1(1-\hat{p}_{Y,1})]|\\
&+\epsilon\\
\leq&\epsilon[\sqrt{\frac{n^{0,0}}{n}}\frac{k^{1,0}_i}{n^{1,0}+1}(p_0+p_{Y,0}) + \sqrt{\frac{n^{0,1}}{n}}\frac{k^{1,1}_i}{n^{1,1}+1}(p_1+p_{Y,1})]+\epsilon^{2}(\frac{n^{0,0}}{n}\frac{k^{1,0}_i}{n^{1,0}+1}+\frac{n^{0,1}}{n}\frac{k^{1,1}_i}{n^{1,1}+1})\\
&+\frac{n^{0,0}-k^{0,0}_i}{n^{0,0}}\sqrt{\frac{n^{0,0}}{n}}\epsilon[\sqrt{\frac{n^{0,0}}{n}}\epsilon+p_0-p_{Y,0}+1]\\
&+\frac{n^{0,1}-k^{0,1}_i}{n^{0,1}}\sqrt{\frac{n^{0,1}}{n}}\epsilon[\sqrt{\frac{n^{0,1}}{n}}\epsilon+p_1-p_{Y,1}+1]+\epsilon\\
\leq&\epsilon[\sqrt{\frac{n^{0,0}}{n}}(p_0+p_{Y,0}) + \sqrt{\frac{n^{0,1}}{n}}(p_1+p_{Y,1})]+\epsilon^{2}(\frac{n^{0,0}}{n}+\frac{n^{0,1}}{n})\\
&+\sqrt{\frac{n^{0,0}}{n}}\epsilon[\sqrt{\frac{n^{0,0}}{n}}\epsilon+p_0-p_{Y,0}+1]+\sqrt{\frac{n^{0,1}}{n}}\epsilon[\sqrt{\frac{n^{0,1}}{n}}\epsilon+p_1-p_{Y,1}+1]+\epsilon\\
\leq&2\epsilon+\epsilon^{2}+\epsilon^{2}+3\epsilon+\epsilon\\
=&2\epsilon^{2}+6\epsilon.
\end{aligned}
$$

Thus we complete the proof.
\end{proof}

\subsection{Theorem for the original candidate set $K$}
\label{AP3}
Sometimes we would use the original candidate set $K$ instead of the small set $K^{\prime}$ to achieve the optimal accuracy more precisely. Now we provide our results for the candidate set $K$.

We first prove Lemma \ref{a2}.


\begin{proof}[Proof of Lemma \ref{a2}]
Since the domain of $q(x)$ is a closed set and $q(x)$ is continuous, we know that $q(x)$ is uniformly continuous. Thus we can easily find $F_{(+)}$ to satisfy the RHS. For $F_{(-)}$, we simply define $F_{(-)}(t) = \inf\limits_{x}\{q(x+t) - q(t)\}$. Since $q(x+t) - q(t) > 0$ for $t>0$ and the domain of $x$ is a closed set, we have $F_{(-)}(\epsilon) > 0$ for $\epsilon>0$ and $\mathop{lim}\limits_{\epsilon \rightarrow 0} F_{(-)}(\epsilon) = 0$. Now we complete the proof.
\end{proof}

Now we provide the following theorem.

\begin{theorem}
Given $\alpha^{\prime} < \alpha$. If $\min\{n^{1,0}, n^{1,1}\} \geq \lceil \frac{\log \frac{\delta}{2}} {\log (1-\alpha)}\rceil$. Suppose $\hat{\phi}$ is the final output of FaiREE, we have:\\
(1) $|DEOO(\hat\phi)| < \alpha$ with probability $(1-\delta)^{M}$, where $M$ is the size of the candidate set.\\
(2) Suppose the density distribution functions of $f^*$ under $A=a, Y=1$ are continuous. When the input classifier $f$ satisfies $\mid f(x,a) -  f^*(x,a) \mid \leq \epsilon_0$, for any $\epsilon > 0$ such that $F^*_{(+)}(\epsilon) \leq \frac{\alpha - \alpha^{\prime}}{2} - F^*_{(+)}(2\epsilon_0)$, we have
$$
\bP(\hat{\phi}(x,a) \neq Y) - \bP(\phi_{\alpha^{\prime}}^*(x,a) \neq Y) \leq  2F^*_{(+)}(\epsilon) +2F^*_{(+)}(2\epsilon_0)+ 4\epsilon^2+12\epsilon
$$
with probability $1 - (2M+4)(e^{-2n^{0,0}\epsilon^2}+e^{-2n^{0,1}\epsilon^2}) - (2M+2)(e^{-2n^{1,0}\epsilon^2}+e^{-2n^{1,1}\epsilon^2}) - (1-F^{1,0}_{(-)}(2\epsilon))^{n^{1,0}} - (1-F^{1,1}_{(-)}(2\epsilon))^{n^{1,1}}$.
\label{p50}
\end{theorem}

\begin{proof}[Proof of Theorem \ref{p50}]
    The (1) of the theorem is a direct corollary from Theorem \ref{p3}, now we prove the (2) of the theorem. The proof can be divided into two parts. The first part is to prove that there exist classifiers in our candidate set that are close to the fair Bayes-optimal classifier. The second part is to prove that our algorithm can successfully choose one of these classifiers with high probability.
    
    We suppose the fair Bayes optimal classifier has the form $\phi_{\alpha^{\prime}}^*(x,a)=\1\{f^*(x,a)>\lambda^*_a\}$. And the output classifier of our algorithm is of the form $\hat\phi(x,a) = \1\{f(x,a)>\lambda_a\}$.
    
    For the first part, for any $\epsilon > 0$, from Lemma \ref{a2}, $t^{1,a}$ has a positive probability $F^{1,a}_{(+)}(2\epsilon)$ to fall in the interval $[\lambda_a^*-\epsilon, \lambda_a^*+\epsilon]$, which implies that the probability that there exists $a \in \{0,1\}$ such that all $t^{1,a}$'s fall out of $[\lambda_a^*-\epsilon, \lambda_a^*+\epsilon]$ is less than $(1-F^{1,0}_{(+)}(2\epsilon))^{n^{1,0}} + (1-F^{1,1}_{(+)}(2\epsilon))^{n^{1,1}}$. So with probability $1- (1-F^{1,0}_{(+)}(2\epsilon))^{n^{1,0}} - (1-F^{1,1}_{(-)}(2\epsilon))^{n^{1,1}}$, there will exist $t^{1,a}$ in $[\lambda_a^*-\epsilon, \lambda_a^*+\epsilon]$, which we denote as $\phi_0(x,a)  = \1\{f(x,a) > t^{1,a}_*\}$. We also denote $\phi_0^*(x,a)  = \1\{f^*(x,a) > t^{1,a}_*\}$. Hence the gap between the classifier  $\phi_0$ and the Bayes-optimal classifier will be very close. In detail, we have
    $$
    \begin{aligned}
    &\mid \bP(\phi_0(x,a) \neq Y) - \bP(\phi_{\alpha^{\prime}}^*(x,a) \neq Y) \mid\\
    \leq &\mid \bP(\phi_0(x,a) \neq Y) - \bP(\phi_0^*(x,a) \neq Y) \mid + \mid \bP(\phi_0^*(x,a) \neq Y) - \bP(\phi_{\alpha^{\prime}}^*(x,a) \neq Y) \mid\\
    \leq &\bP(t^{1,a}_* - \epsilon_0 \leq f^*(x,a) \leq t^{1,a}_* + \epsilon_0) + \bP(\min\{t^{1,a}_*, \lambda_a^*\} \leq f^*(x,a) \leq max\{t^{1,a}_*, \lambda_a^*\}) \\
    (\text{Lemma \ref{a2}})\leq &F^*_{(+)}(2\epsilon_0) + F^*_{(+)}(max\{t^{1,a}_*, \lambda_a^*\} - \min\{t^{1,a}_*, \lambda_a^*\}) \\
    \leq &F^*_{(+)}(2\epsilon_0) + 2F^*_{(+)}(\epsilon)
    \end{aligned}
    $$
    so we complete the first part of the proof.
    
    Now we come to the second part. First, we notice that $DEOO(\phi_0)$ and $DEOO(\phi_{\alpha^{\prime}}^*)$ are close to each other. \\
    $$
    \begin{aligned}
    &\mid \mid DEOO(\phi_0)\mid - \mid DEOO(\phi_{\alpha^{\prime}}^*)\mid \mid \\
    \leq &\mid \mid DEOO(\phi_0)\mid - \mid DEOO(\phi_0^*)\mid \mid + \mid \mid DEOO(\phi_0^*)\mid - \mid DEOO(\phi_{\alpha^{\prime}}^*)\mid \mid\\
    = &\mid \mid \bP(f > t^{1,0}_{*} \mid Y=1, A=0) - \bP(f > t^{1,1}_{*} \mid Y=1, A=1)\mid\\ 
    &- \mid \bP(f^* > t^{1,0}_{*} \mid Y=1, A=0)  - \bP(f^* > t^{1,1}_{*} \mid Y=1, A=1)\mid \mid\\
    &+\mid \mid \bP(f^* > t^{1,0}_{*} \mid Y=1, A=0) - \bP(f^* > t^{1,1}_{*} \mid Y=1, A=1)\mid\\ 
    &- \mid \bP(f^* > \lambda_0^* \mid Y=1, A=0)  - \bP(f^* > \lambda_1^* \mid Y=1, A=1)\mid \mid\\
    \leq &\mid \bP(f > t^{1,0}_{*} \mid Y=1, A=0) - \bP(f^* > t^{1,0}_{*} \mid Y=1, A=0)\mid\\ 
    &+ \mid \bP(f > t^{1,1}_{*} \mid Y=1, A=1)  - \bP(f^* > t^{1,1}_{*} \mid Y=1, A=1)\mid \\
    &+\mid \mid \bP(f^* > t^{1,0}_{*} \mid Y=1, A=0) - \bP(f^* > t^{1,1}_{*} \mid Y=1, A=1)\mid\\ 
    &- \mid \bP(f^* > \lambda_0^* \mid Y=1, A=0)  - \bP(f^* > \lambda_1^* \mid Y=1, A=1)\mid \mid\\
    \leq &  \bP(t^{1,0}_* - \epsilon_0 \leq f^*(x,a) \leq t^{1,0}_* + \epsilon_0) + \bP(t^{1,1}_* - \epsilon_0 \leq f^*(x,a) \leq t^{1,1}_* + \epsilon_0) \\
    &+\mid \bP(f^* > t^{1,0}_{*} \mid Y=1, A=0) - \bP(f^* > t^{1,1}_{*} \mid Y=1, A=1) \\
    &- \bP(f^* > \lambda_0^* \mid Y=1, A=0)  + \bP(f^* > \lambda_1^* \mid Y=1, A=1) \mid\\
    \leq &2F^*_{(+)}(2\epsilon_0) + \bP(\min\{t^{1,a}_*, \lambda_a^*\} \leq f^*(x,a) \leq max\{t^{1,a}_*, \lambda_a^*\}) \\
    (\text{Lemma \ref{a2}})\leq &2F^*_{(+)}(2\epsilon_0)+F^*_{(+)}(max\{t^{1,a}_*, \lambda_a^*\} - \min\{t^{1,a}_*, \lambda_a^*\}) \\
    \leq &2F^*_{(+)}(2\epsilon_0)+2F^*_{(+)}(\epsilon)
    \end{aligned}
    $$
    Thus, $\mid DEOO(\phi_0)\mid \leq \mid DEOO(\phi_{\alpha^{\prime}}^*)\mid + 2F^*_{(+)}(2\epsilon_0) + 2F^*_{(+)}(\epsilon) = \alpha^{\prime}  +2F^*_{(+)}(2\epsilon_0)+2F^*_{(+)}(\epsilon)$. If $F^*_{(+)}(\epsilon) \leq \frac{\alpha - \alpha^{\prime}}{2}-F^*_{(+)}(2\epsilon_0)$, then there will exist at least one feasible classifier in the candidate set.
    
    From Lemma \ref{lemmamis}, we have the mis-classification error 
    
    $\mid \bP(\hat{\phi_i}(x,a) \neq Y) -[ \frac{k^{1,0}_i}{n^{1,0}+1}p_0p_{Y,0} + \frac{k^{1,1}_i}{n^{1,1}+1}p_1p_{Y,1}
+ \frac{n^{0,0}-k^{0,0}_i}{n^{0,0}}p_0(1-p_{Y,0}) + \frac{n^{0,1}-k^{0,1}_i}{n^{0,1}}p_1(1-p_{Y,1})]\mid \leq \epsilon$ with probability $1-2e^{-n^{0,1}\epsilon^2}-2e^{-n^{0,0}\epsilon^2}-2e^{-n^{1,1}\epsilon^2}-2e^{-n^{1,0}\epsilon^2}$.
    

If we can accurately estimate the mis-classification error, than the second part is almost done. 

It's easy to estimate $p_a$ and $p_{Y,a}$ with $\hat{p}_a=\frac{n^{1,a}+n^{0,a}}{n}$ and $\hat{p}_{Y,a}=\frac{n^{Y,a}}{n}$ ($n$ is the number of the total samples). Here, $\frac{n^{1,a}+n^{0,a}}{n} = \frac{\sum\limits_{i=1}^{n} \1\{Z_i^a = 1\}}{n}$ and $\frac{n^{Y,a}}{n} = \frac{\sum\limits_{i=1}^{n} \1\{Z_i^{Y,a} = 1\}}{n}$, where $Z_i^a \sim B(1, p_a)$ and $Z_i^{Y,a} \sim B(1, p_{Y,a})$. From Hoeffding's inequality, we have:
$$
\bP(\mid \hat{p}_a - p_a \mid \geq \sqrt{\frac{n^{0,a}}{n}}\epsilon) \leq 2e^{-2n^{0,a}\epsilon^2},
$$
$$
\bP(\mid \hat{p}_{Y,a} - p_{Y,a} \mid \geq \sqrt{\frac{n^{0,a}}{n}}\epsilon) \leq 2e^{-2n^{0,a}\epsilon^2}
$$


Thus, with probability $1 - 4e^{-2n^{0,a}\epsilon^2}$, we have:
$$
\left\{
       \begin{aligned}
\mid \hat{p}_a - p_a \mid &\leq \sqrt{\frac{n^{0,a}}{n}}\epsilon\\
\mid \hat{p}_{Y,a} - p_{Y,a} \mid &\leq \sqrt{\frac{n^{0,a}}{n}}\epsilon
       \end{aligned}
   \right.\\
$$

Hence, we have with probability $1 - (2M+4)(e^{-2n^{0,0}\epsilon^2}+e^{-2n^{0,1}\epsilon^2}) - (2M+2)(e^{-2n^{1,0}\epsilon^2}+e^{-2n^{1,1}\epsilon^2})$, for each $i \in \{1,\ldots,M\}$,
$$
\begin{aligned}
&|\bP(\hat{\phi_i}(x,a) \neq Y) - \hat{\bP}(\hat{\phi_i}(x,a) \neq Y)|\\
\leq&|\frac{k^{1,0}_i}{n^{1,0}+1}p_0p_{Y,0} + \frac{k^{1,1}_i}{n^{1,1}+1}p_1p_{Y,1}
+ \frac{n^{0,0}-k^{0,0}_i}{n^{0,0}}p_0(1-p_{Y,0}) + \frac{n^{0,1}-k^{0,1}_i}{n^{0,1}}p_1(1-p_{Y,1}) \\
&-[\frac{k^{1,0}_i}{n^{1,0}+1}\hat{p}_0\hat{p}_{Y,0} + \frac{k^{1,1}_i}{n^{1,1}+1}\hat{p}_1\hat{p}_{Y,1}
+ \frac{n^{0,0}-k^{0,0}_i}{n^{0,0}}\hat{p}_0(1-\hat{p}_{Y,0}) + \frac{n^{0,1}-k^{0,1}_i}{n^{0,1}}\hat{p}_1(1-\hat{p}_{Y,1})]|\\
&+\epsilon\\
\leq&\epsilon[\sqrt{\frac{n^{0,0}}{n}}\frac{k^{1,0}_i}{n^{1,0}+1}(p_0+p_{Y,0}) + \sqrt{\frac{n^{0,1}}{n}}\frac{k^{1,1}_i}{n^{1,1}+1}(p_1+p_{Y,1})]+\epsilon^2(\frac{n^{0,0}}{n}\frac{k^{1,0}_i}{n^{1,0}+1}+\frac{n^{0,1}}{n}\frac{k^{1,1}_i}{n^{1,1}+1})\\
&+\frac{n^{0,0}-k^{0,0}_i}{n^{0,0}}\sqrt{\frac{n^{0,0}}{n}}\epsilon[\sqrt{\frac{n^{0,0}}{n}}\epsilon+p_0-p_{Y,0}+1]\\
&+\frac{n^{0,1}-k^{0,1}_i}{n^{0,1}}\sqrt{\frac{n^{0,1}}{n}}\epsilon[\sqrt{\frac{n^{0,1}}{n}}\epsilon+p_1-p_{Y,1}+1]+\epsilon\\
\leq&2\epsilon+\epsilon^2+\epsilon^2+3\epsilon+\epsilon\\
=&2\epsilon^2+6\epsilon
\end{aligned}
$$



Combining two parts together, we have:\\
with probability $1 - (2M+4)(e^{-2n^{0,0}\epsilon^2}+e^{-2n^{0,1}\epsilon^2}) - (2M+2)(e^{-2n^{1,0}\epsilon^2}+e^{-2n^{1,1}\epsilon^2}) - (1-F^{1,0}_{(-)}(2\epsilon))^{n^{1,0}} - (1-F^{1,1}_{(-)}(2\epsilon))^{n^{1,1}}$,
$$
\mid \bP(\hat{\phi}(x,a) \neq Y) - \bP(\phi_{\alpha^{\prime}}^*(x,a) \neq Y) \mid \leq 2F^*_{(+)}(\epsilon) +2F^*_{(+)}(2\epsilon_0)+ 4\epsilon^2+12\epsilon.
$$

Now we complete the proof.
\end{proof}

\subsection{Proof of Theorem \ref{p5}}

\begin{proof}
The (1) of the theorem is a direct corollary from Theorem \ref{p3}, now we prove the (2) of the theorem.

It's sufficient to modify the first part of the proof of Theorem \ref{p50} and the second part simply follows Theorem \ref{p50}.

For the first part, for any $\epsilon > 0$, from Lemma \ref{a2}, $t^{1,0}$ has a positive probability $F^{1,0}_{(-)}(2\epsilon)$ to fall in the interval $[\lambda_0^*-\epsilon, \lambda_0^*+\epsilon]$, which implies that the probability that all $t^{1,0}$'s fall out of $[\lambda_0^*-\epsilon, \lambda_0^*+\epsilon]$ is less than $(1-F^{1,0}_{(-)}(2\epsilon))^{n^{1,0}} $. So with probability $1- (1-F^{1,0}_{(-)}(2\epsilon))^{n^{1,0}}$, there will exist $t^{1,0}$ in $[\lambda_0^*-\epsilon, \lambda_0^*+\epsilon]$, which we denote as $\lambda_0$. We denote the corresponding classifier as $\1\{f(x,a) > \lambda_a\}$. From the proof of Theorem \ref{p50}, we have with probability $1 - 4e^{-2n^{0,a}\epsilon^2} - (1-F^{1,0}_{(-)}(2\epsilon))^{n^{1,0}}$, 
$$
\left\{
\begin{aligned}
\mid \hat{p}_a - p_a \mid &\leq \sqrt{\frac{n^{0,a}}{n}}\epsilon\\
\mid \hat{p}_{Y,a} - p_{Y,a} \mid &\leq \sqrt{\frac{n^{0,a}}{n}}\epsilon\\
\mid \lambda_0 - \lambda_0^* \mid &\leq \epsilon
\end{aligned}
\right.\\
$$
We have the following equalities:\\
$$
\left\{
\begin{aligned}
\lambda_1^*=\frac{1}{2-\frac{(\frac{1}{\lambda_0^*}-2)p_0p_{Y,0}}{p_1p_{Y,1}}}\\
\lambda_1=\frac{1}{2-\frac{(\frac{1}{\lambda_0}-2)\hat{p}_0\hat{p}_{Y,0}}{\hat{p}_1\hat{p}_{Y,1}}}
\end{aligned}
\right.\\
$$
Hence,
$$
\left\{
\begin{aligned}
\frac{p_0p_{Y,0}}{\lambda_0^*}+\frac{p_1p_{Y,1}}{\lambda_1^*} = 2(p_0p_{Y,0}+p_1p_{Y,1})\\
\frac{\hat{p}_0\hat{p}_{Y,0}}{\lambda_0}+\frac{\hat{p}_1\hat{p}_{Y,1}}{\lambda_1} = 2(\hat{p}_0\hat{p}_{Y,0}+\hat{p}_1\hat{p}_{Y,1})
\end{aligned}
\right.\\
$$
By subtracting, we have
$$
(2-\frac{1}{\lambda_1})\hat{p}_1\hat{p}_{Y,1} - (2-\frac{1}{\lambda_1^*})p_1p_{Y,1}
= (\frac{1}{\lambda_0}-2)\hat{p}_0\hat{p}_{Y,0} - (\frac{1}{\lambda_0^*}-2)p_0p_{Y,0}.
$$
We have
$$
\begin{aligned}
&(\frac{1}{\lambda_0}-2)\hat{p}_0\hat{p}_{Y,0} - (\frac{1}{\lambda_0^*}-2)p_0p_{Y,0}\\
\leq&(\frac{1}{\lambda_0}-2)(p_0+\sqrt{\frac{n^{0,0}}{n}}\epsilon)(p_{Y,0}+\sqrt{\frac{n^{0,0}}{n}}\epsilon)-(\frac{1}{\lambda_0^*}-2)p_0p_{Y,0}\\
\leq &\frac{\epsilon}{\lambda_0^*(\lambda_0^*-\epsilon)}p_0p_{Y,0}+\sqrt{\frac{n^{0,0}}{n}}\epsilon(\frac{1}{\lambda_0^*-\epsilon}-2)(p_0+p_{Y,0}+\sqrt{\frac{n^{0,0}}{n}}\epsilon)
\end{aligned}
$$
Similarly, we have
$$
\begin{aligned}
&(\frac{1}{\lambda_0}-2)\hat{p}_0\hat{p}_{Y,0} - (\frac{1}{\lambda_0^*}-2)p_0p_{Y,0}\\
\geq -&\frac{\epsilon}{\lambda_0^*(\lambda_0^*+\epsilon)}p_0p_{Y,0}+\sqrt{\frac{n^{0,0}}{n}}\epsilon(\frac{1}{\lambda_0^*+\epsilon}-2)(-p_0-p_{Y,0}+\sqrt{\frac{n^{0,0}}{n}}\epsilon);
\end{aligned}
$$
$$
\begin{aligned}
&(2-\frac{1}{\lambda_1})\hat{p}_1\hat{p}_{Y,1} - (2-\frac{1}{\lambda_1^*})p_1p_{Y,1}\\
\leq &(\frac{1}{\lambda_1^*}-\frac{1}{\lambda_1})p_1p_{Y,1}+\sqrt{\frac{n^{0,1}}{n}}\epsilon(2-\frac{1}{\lambda_1})(p_1+p_{Y,1}+\sqrt{\frac{n^{0,1}}{n}}\epsilon);
\end{aligned}
$$
$$
\begin{aligned}
&(2-\frac{1}{\lambda_1})\hat{p}_1\hat{p}_{Y,1} - (2-\frac{1}{\lambda_1^*})p_1p_{Y,1}\\
\geq &(\frac{1}{\lambda_1^*}-\frac{1}{\lambda_1})p_1p_{Y,1}+\sqrt{\frac{n^{0,1}}{n}}\epsilon(2-\frac{1}{\lambda_1})(-p_1-p_{Y,1}+\sqrt{\frac{n^{0,1}}{n}}\epsilon).
\end{aligned}
$$
Now, we get:
$$
\left\{
\begin{aligned}
&(\frac{1}{\lambda_1^*}-\frac{1}{\lambda_1})p_1p_{Y,1}+\sqrt{\frac{n^{0,1}}{n}}\epsilon(2-\frac{1}{\lambda_1})(-p_1-p_{Y,1}+\sqrt{\frac{n^{0,1}}{n}}\epsilon)\\
\leq &\frac{\epsilon}{\lambda_0^*(\lambda_0^*-\epsilon)}p_0p_{Y,0}+\sqrt{\frac{n^{0,0}}{n}}\epsilon(\frac{1}{\lambda_0^*-\epsilon}-2)(p_0+p_{Y,0}+\sqrt{\frac{n^{0,0}}{n}}\epsilon)\\
&(\frac{1}{\lambda_1^*}-\frac{1}{\lambda_1})p_1p_{Y,1}+\sqrt{\frac{n^{0,1}}{n}}\epsilon(2-\frac{1}{\lambda_1})(p_1+p_{Y,1}+\sqrt{\frac{n^{0,1}}{n}}\epsilon)\\
\geq &-\frac{\epsilon}{\lambda_0^*(\lambda_0^*+\epsilon)}p_0p_{Y,0}+\sqrt{\frac{n^{0,0}}{n}}\epsilon(\frac{1}{\lambda_0^*+\epsilon}-2)(-p_0-p_{Y,0}+\sqrt{\frac{n^{0,0}}{n}}\epsilon)
\end{aligned}
\right.
$$
Hence, we have:
\small{
$$
\begin{aligned}
&\lambda_1-\lambda_1^* \\
\leq &\frac{\lambda_1\lambda_1^*}{p_1p_{Y,1}}[\frac{\epsilon}{\lambda_0^*(\lambda_0^*-\epsilon)}p_0p_{Y,0}+\sqrt{\frac{n^{0,0}}{n}}\epsilon(\frac{1}{\lambda_0^*-\epsilon}-2)(p_0+p_{Y,0}+\sqrt{\frac{n^{0,0}}{n}}\epsilon) - 2\sqrt{\frac{n^{0,1}}{n}}\epsilon(-p_1-p_{Y,1}+\sqrt{\frac{n^{0,1}}{n}}\epsilon)]\\
\text{and}\\
&\lambda_1-\lambda_1^* \\
\geq &\frac{\lambda_1\lambda_1^*}{p_1p_{Y,1}}[-\frac{\epsilon}{\lambda_0^*(\lambda_0^*+\epsilon)}p_0p_{Y,0}+\sqrt{\frac{n^{0,0}}{n}}\epsilon(\frac{1}{\lambda_0^*+\epsilon}-2)(-p_0-p_{Y,0}+\sqrt{\frac{n^{0,0}}{n}}\epsilon)-2\sqrt{\frac{n^{0,1}}{n}}\epsilon(p_1+p_{Y,1}+\sqrt{\frac{n^{0,1}}{n}}\epsilon)]
\end{aligned}
$$}
Thus, we have:
$$
\mid \lambda_1-\lambda_1^* \mid \leq \frac{\frac{\epsilon}{\lambda_0^*(\lambda_0^*-\epsilon)}p_0p_{Y,0}+\epsilon(\frac{1}{\lambda_0^*-\epsilon}-2)(2+\epsilon)+4\epsilon}{p_1p_{Y,1}}
$$
Now, combined with above, we have with probability $1 - 4e^{-2n^{0,a}\epsilon^2} - (1-F^{1,0}_{(-)}(2\epsilon))^{n^{1,0}}$, 
$$
\left\{
\begin{aligned}
\mid \hat{p}_a - p_a \mid &\leq \sqrt{\frac{n^{0,a}}{n}}\epsilon\\
\mid \hat{p}_{Y,a} - p_{Y,a} \mid &\leq \sqrt{\frac{n^{0,a}}{n}}\epsilon\\
\mid \lambda_0 - \lambda_0^* \mid &\leq \epsilon\\
\mid \lambda_1-\lambda_1^* \mid &\leq \frac{\frac{\epsilon}{\lambda_0^*(\lambda_0^*-\epsilon)}p_0p_{Y,0}+\epsilon(\frac{1}{\lambda_0^*-\epsilon}-2)(2+\epsilon)+4\epsilon}{p_1p_{Y,1}}
\end{aligned}
\right.\\
$$

From the proof of Theorem \ref{p50}, we have:\\
If $F^*_{(+)}(\epsilon) \leq \frac{\alpha - \alpha^{\prime}}{2}  - F^*_{(+)}(2\epsilon_0)$, then with probability $1 - (2M+4)(e^{-2n^{0,0}\epsilon^2}+e^{-2n^{0,1}\epsilon^2})  - (2M+2)(e^{-2n^{1,0}\epsilon^2}+e^{-2n^{1,1}\epsilon^2})- (1-F^{1,0}_{(-)}(2\epsilon))^{n^{1,0}}$,
$$
\begin{aligned}
&\bP(\hat{\phi}(x,a) \neq Y) - \bP(\phi_{\alpha^{\prime}}^*(x,a) \neq Y)\\
\leq &2F^*_{(+)}(2\epsilon_0)+F^*_{(+)}(\epsilon) + F^*_{(+)}(\frac{\frac{\epsilon}{\lambda_0^*(\lambda_0^*-\epsilon)}p_0p_{Y,0}+\epsilon(\frac{1}{\lambda_0^*-\epsilon}-2)(2+\epsilon)+4\epsilon}{p_1p_{Y,1}}) + 4\epsilon^2+12\epsilon.
\end{aligned}
$$

Now we complete the proof.
\end{proof}

\subsection{Fair Bayes-optimal Classifiers under Equalized Odds}

\begin{theorem}[Fair Bayes-optimal Classifiers under Equalized Odds]
    Let $EO^{\star}=\mathrm{DEO}\left(f^{\star}\right) = (E^{\star}, P^{\star})$. For any $\alpha>0$, there exist $0<\alpha_1 \leq \alpha$ and $0<\alpha_2 \leq \alpha$ such that all fair Bayes-optimal classifiers $f_{EO, \alpha}^{\star}$ under the fairness constraint $|\mathrm{DEO}(f)| \preceq (\alpha_1, \alpha_2)$ are given as below:
    \begin{itemize}
    \item When  $\left|EO^{\star}\right| \preceq (\alpha_1, \alpha_2)$, $f_{EO, \alpha}^{\star}=f^{\star}$.
    \item When  $E^{*}>\alpha_1$ or $P^{*}>\alpha_2$, for all $x \in \mathcal{X}$ and $a \in \mathcal{A}$, there exist $t_{1, EO, \alpha}^{\star}$ and $t_{2, EO, \alpha}^{\star}$ such that
    
    $f_{EO, \alpha}^{\star}(x, a)=\\I\left(\eta_{a}(x)>\frac{p_{a} p_{Y, a} + (2a - 1) \frac{P_{Y,a}}{1-P_{Y,a}}t_{2, EO, \alpha}^{\star}}{2 p_{a} p_{Y, a}+(2a - 1) (\frac{P_{Y,a}}{1-P_{Y,a}}t_{2, EO, \alpha}^{\star} - t_{1, EO, \alpha}^{\star})}\right)\\+a \tau_{EO, \alpha}^{\star} I\left(\eta_{a}(x)=\frac{p_{a} p_{Y, a} + (2a - 1) \frac{P_{Y,a}}{1-P_{Y,a}}t_{2, EO, \alpha}^{\star}}{2 p_{a} p_{Y, a}+(2a - 1) (\frac{P_{Y,a}}{1-P_{Y,a}}t_{2, EO, \alpha}^{\star} - t_{1, EO, \alpha}^{\star})}\right),$
    
    Here, we assume $P_{X \mid A=1, Y=1}\left(\eta_{1}(X)=\frac{p_{1} p_{Y, 1} + \frac{p_{Y,1}}{1-p_{Y,1}}t_{2, EO, \alpha}^{\star}}{2p_{1} p_{Y, 1}+\frac{p_{Y,1}}{1-p_{Y,1}}t_{2, EO, \alpha}^{\star}-t_{1, EO, \alpha}^{\star}}\right)\\=P_{X \mid A=0, Y=0}\left(\eta_{1}(X)=\frac{p_{1} p_{Y, 1} + \frac{p_{Y,1}}{1-p_{Y,1}}t_{2, EO, \alpha}^{\star}}{2p_{1} p_{Y, 1}+\frac{p_{Y,1}}{1-p_{Y,1}}t_{2, EO, \alpha}^{\star}-t_{1, EO, \alpha}^{\star}}\right)=0$ and thus $\tau_{EO, \alpha}^{\star} \in[0,1]$  can be an arbitrary constant.
    \end{itemize}
\label{t1}
\end{theorem}

To prove Theorem \ref{t1}, we first introduce the Neyman-Pearson Lemma.
\begin{lemma}
 (Generalized Neyman-Pearson lemma). Let  $f_{0}, f_{1}, $\ldots$, f_{m}$  be  $m+1$  real-valued functions defined on a Euclidean space  $\mathcal{X}$. Assume they are  $\nu$ -integrable for a  $\sigma$ -finite measure  $\nu$. Let  $\phi_{0}$  be any function of the form
$$
\phi_{0}(x)=\left\{\begin{array}{ll}
1, & f_{0}(x)>\sum_{i=1}^{m} c_{i} f_{i}(x) \\
\gamma(x) & f_{0}(x)=\sum_{i=1}^{m} c_{i} f_{i}(x) \\
0, & f_{0}(x)<\sum_{i=1}^{m} c_{i} f_{i}(x)
\end{array}\right.
$$
where  $0 \leq \gamma(x) \leq 1$  for all  $x \in \mathcal{X}$. For given constants  $t_{1}, \ldots, t_{m} \in \mathbb{R}$, let  $\mathcal{T}$  be the class of Borel functions  $\phi: \mathcal{X} \mapsto \mathbb{R}$  satisfying
\begin{equation}
    \int_{\mathcal{X}} \phi f_{i} d \nu \leq t_{i}, \quad i=1,2, \ldots, m
    \label{e1}
\end{equation}

and  $\mathcal{T}_{0}$  be the set of  $\phi$ s  in  $\mathcal{T}$ satisfying  (\ref{e1})  with all inequalities replaced by equalities. If  $\phi_{0} \in \mathcal{T}_{0}$, then  $\phi_{0} \in   \underset{\phi \in \mathcal{T}_{0}}{\operatorname{argmax}} \int_{\mathcal{X}} \phi f_{0} d \nu$. Moreover, if $c_{i} \geq 0$  for all  $i=1, \ldots, m$, then  $\phi_{0} \in \underset{\phi \in \mathcal{T}}{\operatorname{argmax}} \int_{\mathcal{X}} \phi f_{0} d \nu $.
\end{lemma}

Then we come to prove the theorem.
\begin{proof}
If  $\left|EO^{\star}\right| \preceq (\alpha, \alpha)$, we are done since  $f^{\star}$ is just our target classifier. Now, we assume $\left|EO^{\star}\right| \preceq (\alpha, \alpha)$ does not hold. Let $f$  be a classifier that gives output $\widehat{Y}=1$ with probability $f(x, a)$ under $X=x$ and $A=a$. The mis-classification error for $f$ is
$$
\begin{aligned}
R(f) &=\bP(\widehat{Y} \neq Y)=1-\bP(\widehat{Y}=1, Y=1)-\bP(\widehat{Y}=0, Y=0) \\
&=\bP(\widehat{Y}=1, Y=0)-\bP(\widehat{Y}=1, Y=1)+\bP(Y=1)
\end{aligned}
$$
Thus, to minimize the mis-classification error is just equivalent to maximize $\bP(\widehat{Y}=1, Y=0)-\bP(\widehat{Y}=1, Y=1)$, which can be expressed as:
$$
\begin{aligned}
&\bP(\widehat{Y}=1, Y=1)-\bP(\widehat{Y}=1, Y=0) \\
=&\bP_{X \mid A=1, Y=1}(\widehat{Y}=1) p_{1} p_{Y, 1}+\bP_{X \mid A=0, Y=1}(\widehat{Y}=1)\left(1-p_{1}\right) p_{Y, 0} \\
&-\bP_{X \mid A=1, Y=0}(\widehat{Y}=1) p_{1}\left(1-p_{Y, 1}\right)-\bP_{X \mid A=0, Y=0}(\widehat{Y}=1)\left(1-p_{1}\right)\left(1-p_{Y, 0}\right) \\
=&p_{1}\left[p_{Y, 1} \int_{\mathcal{X}} f(x, 1)d \bP_{X \mid 1,1}(x)-\left(1-p_{Y, 1}\right) \int_{\mathcal{X}} f(x, 1) d \bP_{X \mid A=1, Y=0}(x)\right] \\
&+\left(1-p_{1}\right)\left[p_{Y, 0} \int_{\mathcal{X}} f(x, 0) d \bP_{X \mid A=0, Y=1}(x)-\left(1-p_{Y, 0}\right) \int_{\mathcal{X}} f(x, 0) d \bP_{X \mid A=0, Y=0}(x)\right]\\
=&\int_{\mathcal{A}}\int_{\mathcal{X}} f(x, a) M(x, a) d \bP_{X}(x) d \bP(a)
\end{aligned}
$$

with
\begin{equation}
    \begin{array}{l}
    M(x, a)=a p_{1}\left[p_{Y, 1} \frac{d \bP_{X \mid A=1, Y=1}(x)}{d \bP_{X}(x)}-\left(1-p_{Y, 1}\right) \frac{d \bP_{X \mid A=1, Y=0}(x)}{d \bP_{X}(x)}\right] \\
    +(1-a) p_{0}\left[p_{Y, 0} \frac{d \bP_{X \mid A=0, Y=1}(x)}{d \bP_{X}(x)}-\left(1-p_{Y, 0}\right) \frac{d \bP_{X \mid A=0, Y=0}(x)}{d \bP_{X}(x)}\right].
\end{array}
\label{1}
\end{equation}
Next, for any classifier $f$, we have,\\
$$
 \begin{aligned}
    DEO(f) &= (\bP_{X \mid A=1, Y=1}(\widehat{Y}=1) - \bP_{X \mid A=0, Y=1}(\widehat{Y}=1), \bP_{X \mid A=1, Y=0}(\widehat{Y}=1) - \bP_{X \mid A=0, Y=0}(\widehat{Y}=1))\\
   &= (\int_{\mathcal{X}} f(x, 1) d \bP_{X \mid A=1, Y=1}(x)-\int_{\mathcal{X}} f(x, 0) d \bP_{X \mid A=0, Y=1}(x),\\
   &\hspace{2em}\int_{\mathcal{X}} f(x, 1) d \bP_{X \mid A=1, Y=0}(x)-\int_{\mathcal{X}} f(x, 0) d \bP_{X \mid A=0, Y=0}(x))\\
   &= (\int_{\mathcal{A}}\int_{\mathcal{X}} f(x, a)H_{E}(x, a) d \bP_{X}(x) d \bP(a),\int_{\mathcal{A}}\int_{\mathcal{X}} f(x, a)H_{P}(x, a) d \bP_{X}(x) d \bP(a))\\
 \end{aligned}
 $$
 with
 \begin{equation}
 \left\{
       \begin{aligned}
          &H_{E}(x, a) =\frac{a d \bP_{X \mid A=1, Y=1}(x)}{p_1 d \bP_{X}(x)}-\frac{(1-a) d \bP_{X \mid A=0, Y=1}(x)}{p_0 d \bP_{X}(x)}\\
          &H_{P}(x, a) =\frac{a d \bP_{X \mid A=1, Y=0}(x)}{p_1 d \bP_{X}(x)}-\frac{(1-a) d \bP_{X \mid A=0, Y=0}(x)}{p_0 d \bP_{X}(x)}.
       \end{aligned}
   \right.
\label{2}
\end{equation}
Since $\mathop{lim}\limits_{t_2 \rightarrow{\infty}}\frac{p_1^2 p_{Y,1} + \frac{p_{Y,1}}{1-p_{Y,1}}t_2}{2p_1^2 p_{Y,1} + \frac{p_{Y,1}}{1-p_{Y,1}}t_2 - t_1}=\mathop{lim}\limits_{t_2 \rightarrow{\infty}}\frac{p_0^2 p_{Y,0} - \frac{p_{Y,0}}{1-p_{Y,0}}t_2}{2p_0^2 p_{Y,0} - \frac{p_{Y,0}}{1-p_{Y,0}}t_2 + t_1}=1$, we have:
$$
\begin{aligned}
&\mathop{lim}\limits_{t_2 \rightarrow{\infty}}\bP_{X\mid A=1, Y=1}(\eta_1(x) > \frac{p_1^2 p_{Y,1} + \frac{p_{Y,1}}{1-p_{Y,1}}t_2}{2p_1^2 p_{Y,1} + \frac{p_{Y,1}}{1-p_{Y,1}}t_2 - t_1})\\
=&\mathop{lim}\limits_{t_2 \rightarrow{\infty}}\bP_{X\mid A=0, Y=1}(\eta_0(x) > \frac{p_0^2 p_{Y,0}  -\frac{p_{Y,0}}{1-p_{Y,0}}t_2}{2p_0^2 p_{Y,0} - \frac{p_{Y,0}}{1-p_{Y,0}}t_2 + t_1})\\
=&0
\end{aligned}
$$
and
$$
\begin{aligned}
&\mathop{lim}\limits_{t_2 \rightarrow{\infty}}\bP_{X\mid A=1, Y=0}(\eta_1(x) > \frac{p_1^2 p_{Y,1} + \frac{p_{Y,1}}{1-p_{Y,1}}t_2}{2p_1^2 p_{Y,1} + \frac{p_{Y,1}}{1-p_{Y,1}}t_2 - t_1})\\
=&\mathop{lim}\limits_{t_2 \rightarrow{\infty}}\bP_{X\mid A=0, Y=0}(\eta_0(x) > \frac{p_0^2 p_{Y,0}  -\frac{p_{Y,0}}{1-p_{Y,0}}t_2}{2p_0^2 p_{Y,0} - \frac{p_{Y,0}}{1-p_{Y,0}}t_2 + t_1})\\
=&0
\end{aligned}
$$

So there exist $t_{1, EO, \alpha}^{\star}$ and $t_{2, EO, \alpha}^{\star}$, such that:\\
$t_{1, EO, \alpha}^{\star}\frac{E^{*}}{|E^{*}|}>0$, $t_{2, EO, \alpha}^{\star}\frac{P^{*}}{|P^{*}|}>0$, and
\small{
$$
    \left\{
       \begin{aligned}
       \frac{E^{\star}}{\left|E^{\star}\right|}[\bP_{X\mid A=1, Y=1}(\eta_1(x) > \frac{p_1^2 p_{Y,1} + \frac{p_{Y,1}}{1-p_{Y,1}}t_{2, EO, \alpha}^{\star}}{2p_1^2 p_{Y,1} + \frac{p_{Y,1}}{1-p_{Y,1}}t_{2, EO, \alpha}^{\star} - t_{1, EO, \alpha}^{\star}}) + \tau \bP_{X\mid A=1, Y=1}(\eta_1(x) = \frac{p_1^2 p_{Y,1} + \frac{p_{Y,1}}{1-p_{Y,1}}t_{2, EO, \alpha}^{\star}}{2p_1^2 p_{Y,1} + \frac{p_{Y,1}}{1-p_{Y,1}}t_{2, EO, \alpha}^{\star} - t_{1, EO, \alpha}^{\star}}) \\- \bP_{X\mid A=0, Y=1}(\eta_0(x) > \frac{p_0^2 p_{Y,0}  -\frac{p_{Y,0}}{1-p_{Y,0}}t_{2, EO, \alpha}^{\star}}{2p_0^2 p_{Y,0} - \frac{p_{Y,0}}{1-p_{Y,0}}t_{2, EO, \alpha}^{\star} + t_{1, EO, \alpha}^{\star}})] = \alpha_1 < \alpha\\
       \frac{P^{\star}}{\left|P^{\star}\right|}[\bP_{X\mid A=1, Y=0}(\eta_1(x) > \frac{p_1^2 p_{Y,1} + \frac{p_{Y,1}}{1-p_{Y,1}}t_{2, EO, \alpha}^{\star}}{2p_1^2 p_{Y,1} + \frac{p_{Y,1}}{1-p_{Y,1}}t_{2, EO, \alpha}^{\star} - t_{1, EO, \alpha}^{\star}}) + \tau \bP_{X\mid A=1, Y=0}(\eta_1(x) = \frac{p_1^2 p_{Y,1} + \frac{p_{Y,1}}{1-p_{Y,1}}t_{2, EO, \alpha}^{\star}}{2p_1^2 p_{Y,1} + \frac{p_{Y,1}}{1-p_{Y,1}}t_{2, EO, \alpha}^{\star} - t_{1, EO, \alpha}^{\star}}) \\- \bP_{X\mid A=0, Y=0}(\eta_0(x) > \frac{p_0^2 p_{Y,0}  -\frac{p_{Y,0}}{1-p_{Y,0}}t_{2, EO, \alpha}^{\star}}{2p_0^2 p_{Y,0} - \frac{p_{Y,0}}{1-p_{Y,0}}t_{2, EO, \alpha}^{\star} + t_{1, EO, \alpha}^{\star}})] = \alpha_2 < \alpha
       \end{aligned}
   \right.
$$}
 We consider the constraint,
 \begin{equation}
 \left\{
       \begin{aligned}
           &\frac{E^{\star}}{\left|E^{\star}\right|}\int_{\mathcal{A}}\int_{\mathcal{X}} f(x, a)H_{E}(x, a) d \bP_{X}(x) d \bP(a) \leq \alpha_1\\
           &\frac{F^{\star}}{\left|F^{\star}\right|}\int_{\mathcal{A}}\int_{\mathcal{X}} f(x, a)H_{P}(x, a) d \bP_{X}(x) d \bP(a)\leq \alpha_2.
       \end{aligned}
   \right.
\label{3}
 \end{equation}
 Let $f$ be the classifier of the form:
 \begin{equation}
     f_{s_1, s_2, \tau}(x, a)=\left\{\begin{array}{ll}
1, & M(x, a)>s_1 \frac{E^{\star}}{\left|E^{\star}\right|} H_{E}(x, a) + s_2 \frac{P^{\star}}{\left|P^{\star}\right|} H_{P}(x, a);\\
a \tau, & M(x, a)=s_1 \frac{E^{\star}}{\left|E^{\star}\right|} H_{E}(x, a) + s_2 \frac{P^{\star}}{\left|P^{\star}\right|} H_{P}(x, a) ; \\
0, & M(x, a)<s_1 \frac{E^{\star}}{\left|E^{\star}\right|} H_{E}(x, a) + s_2 \frac{P^{\star}}{\left|P^{\star}\right|} H_{P}(x, a),
\end{array}\right.
\label{4}
 \end{equation}
From (\ref{1}) \& (\ref{2}), $M(x, a)>s_1 \frac{E^{\star}}{\left|E^{\star}\right|} H_{E}(x, a) + s_2 \frac{P^{\star}}{\left|P^{\star}\right|} H_{P}(x, a)$ is equal to 
$$
\begin{aligned}
&a p_{1}\left[p_{Y, 1} \frac{d \bP_{X \mid A=1, Y=1}(x)}{d \bP_{X}(x)}-\left(1-p_{Y, 1}\right) \frac{d \bP_{X \mid A=1, Y=0}(x)}{d \bP_{X}(x)}\right]\\
    &+(1-a) p_{0}\left[p_{Y, 0} \frac{d \bP_{X \mid A=0, Y=1}(x)}{d \bP_{X}(x)}-\left(1-p_{Y, 0}\right) \frac{d \bP_{X \mid A=0, Y=0}(x)}{d \bP_{X}(x)}\right]\\
    >&\frac{E^{\star}}{\left|E^{\star}\right|}s_1(\frac{a d \bP_{X \mid A=1, Y=1}(x)}{p_1 d \bP_{X}(x)}-\frac{(1-a) d \bP_{X \mid A=0, Y=1}(x)}{p_0 d \bP_{X}(x)}) \\&+ \frac{P^{\star}}{\left|P^{\star}\right|}s_2(\frac{a d \bP_{X \mid A=1, Y=0}(x)}{p_1 d \bP_{X}(x)}-\frac{(1-a) d \bP_{X \mid A=0, Y=0}(x)}{p_0 d \bP_{X}(x)})
\end{aligned}
$$
which is equal to
$$
\left\{
       \begin{aligned}
           &p_{1}\left[p_{Y, 1} \frac{d \bP_{X \mid A=1, Y=1}(x)}{d \bP_{X}(x)}-\left(1-p_{Y, 1}\right) \frac{d \bP_{X \mid A=1, Y=0}(x)}{d \bP_{X}(x)}\right] > \frac{\frac{E^{\star}}{\left|E^{\star}\right|}s_1d \bP_{X \mid A=1, Y=1}(x) + \frac{P^{\star}}{\left|P^{\star}\right|}s_2d \bP_{X \mid A=1, Y=0}(x)}{p_1 d \bP_{X}(x)}&,  a = 1\\
           &p_{0}\left[p_{Y, 0} \frac{d \bP_{X \mid A=0, Y=1}(x)}{d \bP_{X}(x)}-\left(1-p_{Y, 0}\right) \frac{d \bP_{X \mid A=0, Y=0}(x)}{d \bP_{X}(x)}\right] >- \frac{\frac{E^{\star}}{\left|E^{\star}\right|}s_1d \bP_{X \mid A=0, Y=1}(x) + \frac{P^{\star}}{\left|P^{\star}\right|}s_2d \bP_{X \mid A=0, Y=0}(x)}{p_0 d \bP_{X}(x)}&,  a = 0
       \end{aligned}
   \right.$$
Thus, 
$$
\begin{aligned}
&M(x, a)>s_1 \frac{E^{\star}}{\left|E^{\star}\right|} H_{E}(x, a) + s_2 \frac{P^{\star}}{\left|P^{\star}\right|} H_{P}(x, a)\\
\iff &p_{a}\left[p_{Y, a} \frac{d \bP_{X \mid A=a, Y=1}(x)}{d \bP_{X}(x)}-\left(1-p_{Y, a}\right) \frac{d \bP_{X \mid A=a, Y=0}(x)}{d \bP_{X}(x)}\right] > (2a-1)\frac{\frac{E^{\star}}{\left|E^{\star}\right|}s_1d \bP_{X \mid A=a, Y=1}(x) + \frac{P^{\star}}{\left|P^{\star}\right|}s_2d \bP_{X \mid A=a, Y=0}(x)}{p_a d \bP_{X}(x)}.\\
\iff &\frac{p_{Y, a}d \bP_{X \mid A=a, Y=1}(x)}{p_{Y, a}d \bP_{X \mid A=a, Y=1}(x) + (1 - p_{Y, a})d \bP_{X \mid A=a, Y=0}(x)} > \frac{p_{a}^2 p_{Y, a} + (2a-1)\frac{p_{Y,a}}{1-p_{Y,a}}t_2}{2p_{a}^2 p_{Y, a} + (2a-1)(\frac{p_{Y,a}}{1-p_{Y,a}}t_2 - t_1)}.\\
&\text{where } t_1 = 2\frac{E^{\star}}{\left|E^{\star}\right|}s_1,\quad t_2 = 2\frac{P^{\star}}{\left|P^{\star}\right|}s_2.\\
\iff &\eta_a(x) > \frac{p_a^2 p_{Y, a} + (2a-1)\frac{p_{Y,a}}{1-p_{Y,a}}t_2}{2p_a^2 p_{Y, a} + (2a-1)(\frac{p_{Y,a}}{1-p_{Y,a}}t_2 - t_1)}.
\end{aligned}
$$
As a result, $f_{s_1, s_2, \tau}(x, a)$ in (\ref{4}) can be written as
\begin{equation}
    f_{t_1, t_2, \tau}(x, a) = \1\{\eta_a(x) > \frac{p_a^2 p_{Y, a} + (2a-1)\frac{p_{Y,a}}{1-p_{Y,a}}t_2}{2p_a^2 p_{Y, a} + (2a-1)(\frac{p_{Y,a}}{1-p_{Y,a}}t_2 - t_1)}\} + a \tau \1\{\eta_a(x) = \frac{p_a^2 p_{Y, a} + (2a-1)\frac{p_{Y,a}}{1-p_{Y,a}}t_2}{2p_a^2 p_{Y, a} + (2a-1)(\frac{p_{Y,a}}{1-p_{Y,a}}t_2 - t_1)}\}.
\label{5}
\end{equation}
Further, the constraint (\ref{3}) for $f$ in (\ref{5}) is equivalent to
\begin{equation}
    \left\{
       \begin{aligned}
       \frac{E^{\star}}{\left|E^{\star}\right|}[\bP_{X\mid A=1, Y=1}(\eta_1(x) > \frac{p_1^2 p_{Y,1} + \frac{p_{Y,1}}{1-p_{Y,1}}t_2}{2p_1^2 p_{Y,1} + \frac{p_{Y,1}}{1-p_{Y,1}}t_2 - t_1}) + \tau \bP_{X\mid A=1, Y=1}(\eta_1(x) = \frac{p_1^2 p_{Y,1} + \frac{p_{Y,1}}{1-p_{Y,1}}t_2}{2p_1^2 p_{Y,1} + \frac{p_{Y,1}}{1-p_{Y,1}}t_2 - t_1}) \\- \bP_{X\mid A=0, Y=1}(\eta_0(x) > \frac{p_0^2 p_{Y,0}  -\frac{p_{Y,0}}{1-p_{Y,0}}t_2}{2p_0^2 p_{Y,0} - \frac{p_{Y,0}}{1-p_{Y,0}}t_2 + t_1})] \leq \alpha_1\\
       \frac{P^{\star}}{\left|P^{\star}\right|}[\bP_{X\mid A=1, Y=0}(\eta_1(x) > \frac{p_1^2 p_{Y,1} + \frac{p_{Y,1}}{1-p_{Y,1}}t_2}{2p_1^2 p_{Y,1} + \frac{p_{Y,1}}{1-p_{Y,1}}t_2 - t_1}) + \tau \bP_{X\mid A=1, Y=0}(\eta_1(x) = \frac{p_1^2 p_{Y,1} + \frac{p_{Y,1}}{1-p_{Y,1}}t_2}{2p_1^2 p_{Y,1} + \frac{p_{Y,1}}{1-p_{Y,1}}t_2 - t_1}) \\- \bP_{X\mid A=0, Y=0}(\eta_0(x) > \frac{p_0^2 p_{Y,0}  -\frac{p_{Y,0}}{1-p_{Y,0}}t_2}{2p_0^2 p_{Y,0} - \frac{p_{Y,0}}{1-p_{Y,0}}t_2 + t_1})] \leq \alpha_2
       \end{aligned}
   \right.
\label{6}
\end{equation}

Now, let $\mathcal{T}_{\alpha_1, \alpha_2}$ be the class of Borel functions $f$ that satisfy (\ref{3}) and $\mathcal{T}_{\alpha_1, \alpha_2, 0}$ be the set of $f$-s in $\mathcal{T}_{\alpha}$ that satisfy (\ref{3}) with all the inequalities being replaced by equalities.

From the definition of $t_{1, EO, \alpha}^{\star}$ and $t_{2, EO, \alpha}^{\star}$, clearly $f_{t_{1, EO, \alpha}^{\star}, t_{2, EO, \alpha}^{\star}, \tau}(x, a) \in \mathcal{T}_{\alpha_1, \alpha_2, 0}$. Further, we have $s_1^*=t_{1, EO, \alpha}^{\star}\frac{E^{*}}{2|E^{*}|} > 0$ and $s_2^*=t_{2, EO, \alpha}^{\star}\frac{P^{*}}{2|P^{*}|}>0$.

Hence, from Generalized Neyman-Pearson lemma, we have:
$$
f_{t_{1, EO, \alpha}^{\star}, t_{2, EO, \alpha}^{\star}, \tau}(x, a) \in \mathop{argmax}\limits_{f \in \mathcal{T}_{\alpha_1, \alpha_2}}\int_{\mathcal{A}}\int_{\mathcal{X}} f_{t_1,t_2,\tau}(x, a) M(x, a) d \bP_{X}(x) d \bP(a)
$$
Now we complete our proof.
\end{proof}

\subsection{Proof of Proposition \ref{pro:deo}}

\begin{proof}
The classifier is
$$ \phi=\left\{
       \begin{aligned}
           \1\{f(x,0)>t^{1,0}_{(k^{1,0})}\},a=0\\
           \1\{f(x,1)>t^{1,1}_{(k^{1,1})}\},a=1
       \end{aligned}
   \right.
$$
we have:
$$
|DEO(\phi)|= (|F^{1,1}(t^{1,1}_{(k^{1,1})})-F^{1,0}(t^{1,0}_{(k^{1,0})})|, |F^{0,1}(t^{1,1}_{(k^{1,1})})-F^{0,0}(t^{1,0}_{(k^{1,0})})|)
$$

From Proposition \ref{p1}, we have $\bP(|F^{1,1}(t^{1,1}_{(k^{1,1})})-F^{1,0}(t^{1,0}_{(k^{1,0})})|>\alpha) \leq \E[\sum\limits^{n^{1,0}}_{j=k^{1,0}}\left(\begin{aligned}
n&^{1,0} \\
&j
\end{aligned}\right)(Q^{1,1}-\alpha)^{j}(1-(Q^{1,1}-\alpha))^{n^{1,0}-j}]+\E[\sum\limits^{n^{1,1}}_{j=k^{1,1}}\left(\begin{aligned}
n&^{1,1} \\
&j
\end{aligned}\right)(Q^{1,0}-\alpha)^{j}(1-(Q^{1,0}-\alpha))^{n^{1,1}-j}]
$\\
Also, 
$$
\begin{aligned}
&\bP(|F^{0,1}(t^{1,1}_{(k^{1,1})})-F^{0,0}(t^{1,0}_{(k^{1,0})})|>\alpha)\\
=&\bP(F^{0,1}(t^{1,1}_{(k^{1,1})})-F^{0,0}(t^{1,0}_{(k^{1,0})})>\alpha)+\bP(F^{0,1}(t^{1,1}_{(k^{1,1})})-F^{0,0}(t^{1,0}_{(k^{1,0})})<-\alpha)\\
\mathop{=}\limits^{\Delta}&A+B
\end{aligned}
$$

And we have 
$$
\begin{aligned}
A&=\bP(F^{0,1}(t^{1,1}_{(k^{1,1})})-F^{0,0}(t^{1,0}_{(k^{1,0})})>\alpha)\\
&=\bP(F^{0,0}(t^{1,0}_{(k^{1,0})})<F^{0,1}(t^{1,1}_{(k^{1,1})})-\alpha)\\
&\leq \bP(F^{0,0}(t^{0,0}_{(k^{0,0})})<F^{0,1}(t^{0,1}_{(k^{0,1}+1)})-\alpha)\\
&\leq \E[\bP(t^{0,0}_{(k^{0,0})}<{F^{0,0}}^{-1}(F^{0,1}(t^{0,1}_{(k^{0,1}+1)})-\alpha))\1\{F^{0,1}(t^{0,1}_{(k^{0,1}+1)})-\alpha>0\} \mid t^{0,1}_{(k^{0,1}+1)}]\\
&=E\{\bP[\text{at least }k^{0,0}\text{ of }t^{0,0}\text{'s are less than }{F^{0,0}}^{-1}(F^{0,1}(t^{0,1}_{(k^{0,1}+1)})-\alpha) ]\1\{F^{0,1}(t^{0,1}_{(k^{0,1}+1)})-\alpha>0\}\mid t^{0,1}_{(k^{0,1}+1)}\}\\
&=E\{\sum\limits^{n^{0,0}}_{j=k^{0,0}}\bP[\text{exactly j of the }t^{0,0}\text{'s are less than }{F^{0,0}}^{-1}(F^{0,1}(t^{0,1}_{(k^{0,1}+1)})-\alpha)]\1\{F^{0,1}(t^{0,1}_{(k^{0,1}+1)})-\alpha>0\}\mid t^{0,1}_{(k^{0,1}+1)}\}\\
&=E\{\sum\limits^{n^{0,0}}_{j=k^{0,0}}\left(\begin{aligned}
n&^{0,0} \\
&j
\end{aligned}\right)\bP[t^{0,0}<{F^{0,0}}^{-1}(F^{0,1}(t^{0,1}_{(k^{0,1}+1)})-\alpha)]^{j}(1-\bP[t^{0,0}<{F^{0,0}}^{-1}(F^{0,1}(t^{0,1}_{(k^{0,1}+1)})-\alpha)])^{n^{0,0}-j}\\&\text{\space\space}\1\{F^{0,1}(t^{0,1}_{(k^{0,1}+1)})-\alpha>0\}\mid t^{0,1}_{(k^{0,1}+1)}\}\\
&\leq \E[\sum\limits^{n^{0,0}}_{j=k^{0,0}}\left(\begin{aligned}
n&^{0,0} \\
&j
\end{aligned}\right)(F^{0,1}(t^{0,1}_{(k^{0,1}+1)})-\alpha)^{j}(1-(F^{0,1}(t^{0,1}_{(k^{0,1}+1)})-\alpha))^{n^{0,0}-j}\mid t^{0,1}_{(k^{0,1}+1)}]
\end{aligned}
$$

Similarly, we have
$$
B \leq \E[\sum\limits^{n^{0,1}}_{j=k^{0,1}}\left(\begin{aligned}
n&^{0,1} \\
&j
\end{aligned}\right)(F^{0,0}(t^{0,0}_{(k^{0,0}+1)})-\alpha)^{j}(1-(F^{0,0}(t^{0,0}_{(k^{0,0}+1)})-\alpha))^{n^{0,1}-j} \mid t^{0,0}_{(k^{0,0}+1)}]
$$
Hence, we have

$$
\begin{aligned}
A+B \leq &\E[\sum\limits^{n^{0,0}}_{j=k^{0,0}}\left(\begin{aligned}
n&^{0,0} \\
&j
\end{aligned}\right)(F^{0,1}(t^{0,1}_{(k^{0,1}+1)})-\alpha)^{j}(1-(F^{0,1}(t^{0,1}_{(k^{0,1}+1)})-\alpha))^{n^{0,0}-j} \mid t^{0,1}_{(k^{0,1}+1)}]\\
&+\E[\sum\limits^{n^{0,1}}_{j=k^{0,1}}\left(\begin{aligned}
n&^{0,1} \\
&j
\end{aligned}\right)(F^{0,0}(t^{0,0}_{(k^{0,0}+1)})-\alpha)^{j}(1-(F^{0,0}(t^{0,0}_{(k^{0,0}+1)})-\alpha))^{n^{0,1}-j}\mid t^{0,0}_{(k^{0,0}+1)}]
\end{aligned}
$$

Since $F^{0,a}(t^{0,a}_{(k^{0,a}+1)})$ is stochastically dominated by $Beta(k^{0,a}+1, n^{0,a} - k^{0,a})$, we complete the proof.
\end{proof}
\subsection{Algorithms for other group fairness constraints}

\label{A7}
In addition to Equality of Opportunity and Equalized Odds, there are other common fairness constraints and we can extend FaiREE to them.

\begin{definition}[Demographic Parity]
    A classifier satisfies Demographic Parity if its prediction $\widehat{Y}$ is statistically independent of the sensitive attribute $A$ :
$$
\bP(\widehat{Y}=1 \mid A=1)=\bP(\widehat{Y}=1 \mid A=0)
$$
\end{definition}

\begin{definition}[Predictive Equality]
A classifier satisfies Predictive Equality if it achieves the same TNR (or FPR) among protected groups:
$$
\bP_{X \mid A=1, Y=0}(\widehat{Y}=1)=\bP_{X \mid A=0, Y=0}(\widehat{Y}=1)
$$
\end{definition}

\begin{definition}[Equalized Accuracy]
A classifier satisfies Equalized Accuracy if its mis-classification error is statistically independent of the sensitive attribute $A$:
$$
\bP(\widehat{Y} \neq Y \mid A=1)=\bP(\widehat{Y} \neq Y \mid A=0)
$$
\end{definition}


Similar to $DEOO$, we can define the following measures:

    \begin{align}
\mathrm{DDP} &=\bP_{X \mid A=1}(\widehat{Y}=1)-\bP_{X \mid A=0}(\widehat{Y}=1) \label{ddp}\\
\mathrm{DPE} &=\bP_{X \mid A=1, Y=0}(\widehat{Y}=1)-\bP_{X \mid A=0, Y=0}(\widehat{Y}=1)\label{dpe}\\
\mathrm{DEA} &= \bP(\widehat{Y} \neq Y \mid A=1)-\bP(\widehat{Y} \neq Y \mid A=0).\label{dea}
\end{align}

\subsubsection{FaiREE for Demographic Parity}

\begin{algorithm}[!htb]
\caption{FaiREE for Demographic Parity}
\KwIn{\\Data: $S$ = $S^{0,0}\cup S^{0,1} \cup S^{1,0} \cup S^{1,1}$\\ $\alpha$: error bound \\ $\delta$: small tolerance level\\ $f$: a classifier
}

        $T^{y,a}=\{f(x^{y,a}_{1}),\ldots,f(x^{y,a}_{n_{y,a}})\}$\\
    $\{t^{y,a}_{(1)},\ldots,t^{y,a}_{(n_{y,a})}\}= $sort($T^{y,a}$) \\
    $T^{y}=T^{y,0} \cup T^{y,1}$\\
    $\{t^{y}_{(1)},\ldots,t^{y}_{(n_{y})}\} = $sort($T^{y}$)\\
    Define $g(k,a)=\bE[\sum\limits^{n^{a}}_{j=k}{n^{a}\choose j}(Q^{1-a}-\alpha)^{j}(1-(Q^{1-a}-\alpha))^{n^{a}-j}]$ with $Q^{a}\sim Beta(k,n^{a}-k+1)$,
    $L(k^{0},k^{1})=g(k^0, 0) + g(k^1,1)$\\

 Build candidate set 
 $K =  \{(k^{0}, k^{1}) \mid L(k^{0}, k^{1}) \leq \delta\} = \{(k^{0}_1, k^{1}_1), \ldots, (k^{0}_M,k^{1}_M)\}$\\ 

Find $k_{i}^{y,0}$: $t^{y,0}_{(k_{i}^{y,0})} \leq t^{0}_{(k_{i}^{0})} < t^{y,0}_{(k_{i}^{y,0} + 1)}$, $t^{y,1}_{(k_{i}^{y,1})} \leq t^{1}_{(k_{i}^{1})} < t^{y,1}_{(k_{i}^{y,1} + 1)}$, $y \in \{0,1\}$\\
$i_*\leftarrow \mathop{\arg\min}\limits_{i \in [M]} \{\hat{e_i}\}$ ($\hat{e_i}$ is defined in Proposition \ref{p4})\\
\KwOut{$\hat\phi(x,a)=\1\{f(x,a)>t^{a}_{(k^{a}_{i_*})}\}$}
\end{algorithm}

Similar to the algorithm for Equality of Opportunity, we have the following propositions and assumption:

\begin{proposition}
Given $k^{0}, k^{1}$ satisfying $k^{a} \in \{1,\ldots,n^{a}\}$ $(a=0,1)$. Define $\phi(x,a)=\1\{f(x,a)>t^{a}_{(k^{a})}\}$, $g(k,a)=\bE[\sum\limits^{n^{a}}_{j=k}{n^{a}\choose j}(Q^{1-a}-\alpha)^{j}(1-(Q^{1-a}-\alpha))^{n^{a}-j}]$ with $Q^{a}\sim Beta(k,n^{a}-k+1)$, then we have:
\begin{align*}
    \bP(|DDP(\phi)|>\alpha) \leq g(k^{0},0)+g(k^{1},1). 
\end{align*}
If $t^{a}$ is continuous random variable, the equality holds.

\end{proposition}


\begin{theorem}
If $\min\{n^{0}, n^{1}\} \geq \lceil \frac{\log \frac{\delta}{2}} {\log (1-\alpha)}\rceil$, we have $|DDP(\hat\phi_i)| < \alpha$ with probability $1-\delta$, for each $i\in\{1,\ldots,M\}$.
\end{theorem}

\begin{theorem}
Given $\alpha^{\prime} < \alpha$. If $\min\{n^{0}, n^{1}\} \geq \lceil \frac{\log \frac{\delta}{2}} {\log (1-\alpha)}\rceil$. Suppose $\hat{\phi}$ is the final output of FaiREE, we have:\\
(1) $|DDP(\hat\phi)| \leq \alpha$ with probability $(1-\delta)^{M}$, where $M$ is the size of the candidate set.\\
(2) Suppose the density distribution functions of $f^*$ under $A=a, Y=1$ are continuous. $\phi_{DDP,\alpha}^* = {\arg\min}_{\mid DDP(\phi)\mid \leq \alpha} \bP(\phi(x,a) \neq Y)$. When the input classifier $f$ satisfies $\| f(x,a) - f^*(x,a)\|_{\infty} \leq \epsilon_0$, for any $\epsilon > 0$ such that $F^*_{(+)}(\epsilon) \leq \frac{\alpha - \alpha^{\prime}}{2} - F^*_{(+)}(2\epsilon_0)$, we have $$
\mid \bP(\hat{\phi}(x,a) \neq Y) - \bP(\phi_{\alpha^{\prime}}^*(x,a) \neq Y)\mid \leq 2F^*_{(+)}(2\epsilon_0) + 2F^*_{(+)}(\epsilon) + 2\epsilon^3+12\epsilon^2+16\epsilon
$$
with probability $1 - (2M+4)(e^{-2n^{1,0}\epsilon^2}+e^{-2n^{1,1}\epsilon^2}+e^{-2n^{0,0}\epsilon^2}+e^{-2n^{0,1}\epsilon^2})$.
\end{theorem}


\subsubsection{FaiREE for Predictive Equality}

\begin{algorithm}[!htb]
\caption{FaiREE for Predictive Opportunity}
\KwIn{\\Data: $S$ = $S^{0,0}\cup S^{0,1} \cup S^{1,0} \cup S^{1,1}$\\ $\alpha$: error bound \\ $\delta$: small tolerance level\\ $f$: a classifier
}
   $T^{y,a}=\{f(x^{y,a}_{1}),\ldots,f(x^{y,a}_{n_{y,a}})\}$\\
    $\{t^{y,a}_{(1)},\ldots,t^{y,a}_{(n_{y,a})}\}= $sort($T^{y,a}$) \\
    Define $g_{0}(k,a)=\bE[\sum\limits^{n^{0,a}}_{j=k}{n^{0,a}\choose j}(Q^{0,1-a}-\alpha)^{j}(1-(Q^{0,1-a}-\alpha))^{n^{0,a}-j}]$ with $Q^{0,a}\sim Beta(k,n^{0,a}-k+1)$,
    $L(k^{0,0},k^{0,1})= g_{0}(k^{0,0},0)+g_{0}(k^{0,1},1)$\\

Build candidate set 
$K =  \{(k^{0,0}, k^{0,1}) \mid L(k^{0,0}, k^{0,1}) \leq \delta\} = \{(k^{0,0}_1, k^{0,1}_1), \ldots, (k^{0,0}_M,k^{0,1}_M)\}$\\ 

Find $k_{i}^{1,0}, k_{i}^{1,1}$: $t^{1,0}_{(k_{i}^{1,0})} \leq t^{0,0}_{(k_{i}^{0,0})} < t^{1,0}_{(k_{i}^{1,0} + 1)}$, $t^{1,1}_{(k_{i}^{1,1})} \leq t^{0,1}_{(k_{i}^{0,1})} < t^{1,1}_{(k_{i}^{1,1} + 1)}$\\
$i_*\leftarrow \mathop{\arg\min}\limits_{i \in [M]} \{\hat{e_i}\}$ ($\hat{e_i}$ is defined in Proposition \ref{p4})\\
\KwOut{$\hat\phi(x,a)=\1\{f(x,a)>t^{0,a}_{(k^{0,a}_{i_*})}\}$}
\end{algorithm}

Similar to the algorithm for Equality of Opportunity, we have the following propositions and assumption:

\begin{proposition}
Given $k^{0,0}, k^{0,1}$ satisfying $k^{1,a} \in \{1,\ldots,n^{0,a}\}$ $(a=0,1)$. Define $\phi(x,a)=\1\{f(x,a)>t^{0,a}_{(k^{0,a})}\}$, 
$g_{0}(k,a)=\bE[\sum\limits^{n^{0,a}}_{j=k}{n^{0,a}\choose j}(Q^{0,1-a}-\alpha)^{j}(1-(Q^{0,1-a}-\alpha))^{n^{0,a}-j}]$ with $Q^{0,a}\sim Beta(k,n^{0,a}-k+1)$, then we have:
\begin{align*}
    \bP(|DPE(\phi)|>\alpha) \leq g_{0}(k^{0,0},0)+g_{0}(k^{0,1},1). 
\end{align*} 
If $t^{0,a}$ is continuous random variable, the equality holds.
\end{proposition}


\begin{theorem}
If $\min\{n^{0,0}, n^{0,1}\} \geq \lceil \frac{\log \frac{\delta}{2}} {\log (1-\alpha)}\rceil$, we have $|DPE(\hat\phi_i)| < \alpha$ with probability $1-\delta$, for each $i\in\{1,\ldots,M\}$.
\end{theorem}

\begin{theorem}
Given $\alpha^{\prime} < \alpha$. If $\min\{n^{0,0}, n^{0,1}\} \geq \lceil \frac{\log \frac{\delta}{2}} {\log (1-\alpha)}\rceil$. Suppose $\hat{\phi}$ is the final output of FaiREE, we have:\\
(1) $|DPE(\hat\phi)| \leq \alpha$ with probability $(1-\delta)^{M}$, where $M$ is the size of the candidate set.\\
(2) Suppose the density distribution functions of $f^*$ under $A=a, Y=1$ are continuous. $\phi_{DPE,\alpha}^* = {\arg\min}_{\mid DPE(\phi)\mid \leq \alpha} \bP(\phi(x,a) \neq Y)$. When the input classifier $f$ satisfies $\| f(x,a) - f^*(x,a)\|_{\infty} \leq \epsilon_0$, for any $\epsilon > 0$ such that $F^*_{(+)}(\epsilon) \leq \frac{\alpha - \alpha^{\prime}}{2} - F^*_{(+)}(2\epsilon_0)$, we have $$
\mid \bP(\hat{\phi}(x,a) \neq Y) - \bP(\phi_{\alpha^{\prime}}^*(x,a) \neq Y)\mid \leq 2F^*_{(+)}(2\epsilon_0) + 2F^*_{(+)}(\epsilon) + 2\epsilon^3+12\epsilon^2+16\epsilon
$$
with probability $1 - (2M+4)(e^{-2n^{1,0}\epsilon^2}+e^{-2n^{1,1}\epsilon^2}) - (1-F^{0,0}_{(-)}(2\epsilon))^{n^{0,0}} - (1-F^{0,1}_{(-)}(2\epsilon))^{n^{0,1}}$.
\end{theorem}


\subsubsection{FaiREE for Equalized Accuracy}

\begin{algorithm}[!htb]
\caption{FaiREE for Equalized Accuracy}
\KwIn{\\Data: $S$ = $S^{0,0}\cup S^{0,1} \cup S^{1,0} \cup S^{1,1}$\\ $\alpha$: error bound ($\alpha > |p_{Y,1} - p_{Y,0}|$) \\ $\delta$: small tolerance level\\$f$: a classifier
}
    $T^{y,a}=\{f(x^{y,a}_{1}),\ldots,f(x^{y,a}_{n_{y,a}})\}$\\
    $\{t^{y,a}_{(1)},\ldots,t^{y,a}_{(n_{y,a})}\}= $sort($T^{y,a}$) \\
    Define $g_{1}(k,a)=\bE[\sum\limits^{n^{1,a}}_{j=k}{n^{1,a}\choose j}(\frac{p_{Y,1-a}Q^{1,1-a}-\alpha}{p_{Y,a}})^{j}(1-\frac{p_{Y,1-a}Q^{1,1-a}-\alpha}{p_{Y,a}})^{n^{1,a}-j}]$ with $Q^{1,a}\sim Beta(k,n^{1,a}-k+1)$,$L_1(k^{1,0},k^{1,1}) = g_{1}(k^{1,1},1) + g_{1}(k^{1,0},0)$\\

Define $g_{0}(k,a)=\bE[\sum\limits^{n^{0,a}}_{j=k}{n^{0,a}\choose j}(\frac{(1-p_{Y,1-a})Q^{0,1-a}+p_{Y,1-a}-p_{Y,a}-\alpha}{1-p_{Y,0}})^{j}(1-\frac{(1-p_{Y,1-a})Q^{0,1-a}+p_{Y,1-a}-p_{Y,a}-\alpha}{1-p_{Y,a}})^{n^{0,a}-j}]$ with $Q^{0,a}\sim Beta(k+1,n^{0,a}-k)$, $L_0(k^{0,0},k^{0,1}) = g_{0}(k^{0,1},1) + g_{0}(k^{0,0},0)$\\

Find $k^{0,0}, k^{0,1}$: $t^{0,0}_{(k^{0,0})} \leq t^{1,0}_{(k^{1,0})} < t^{0,0}_{(k^{0,0} + 1)}$, $t^{0,1}_{(k^{0,1})} \leq t^{1,1}_{(k^{1,1})} < t^{0,1}_{(k^{0,1} + 1)}$.\\

Build candidate set 
$K =  \{(k^{1,0}, k^{1,1}) \mid L_1(k^{1,0}, k^{1,1}) + L_0(k^{0,0}, k^{0,1}) \leq \delta\} = \{(k^{1,0}_1, k^{1,1}_1), \ldots, (k^{1,0}_M,k^{1,1}_M)\}$\\ 

$i_*\leftarrow \mathop{\arg\min}\limits_{i \in [M]} \{\hat{e_i}\}$ ($\hat{e_i}$ is defined in Proposition \ref{p4})\\
\KwOut{$\hat\phi(x,a)=\1\{f(x,a)>t^{1,a}_{(k^{1,a}_{i_*})}\}$}
\end{algorithm}

Similar to the algorithm for Equalized Odds, we have the following propositions and assumption:

\begin{proposition}
Given $k^{1,0}, k^{1,1}$ satisfying $k^{1,a} \in \{1,\ldots,n^{1,a}\}$ $(a=0,1)$ and $\alpha > |p_{Y,1} - p_{Y,0}|$. Define $\phi(x,a)=\1\{f(x,a)>t^{1,a}_{(k^{1,a})}\}$, 
$g_{1}(k,a)=\bE[\sum\limits^{n^{1,a}}_{j=k}{n^{1,a}\choose j}(\frac{p_{Y,1-a}Q^{1,1-a}-\alpha}{p_{Y,a}})^{j}(1-\frac{p_{Y,1-a}Q^{1,1-a}-\alpha}{p_{Y,a}})^{n^{1,a}-j}]$ with $Q^{1,a}\sim Beta(k,n^{1,a}-k+1)$, $q_{Y,a}(\alpha)=\frac{(1-p_{Y,1-a})Q^{0,1-a}+p_{Y,1-a}-p_{Y,a}-\alpha}{1-p_{Y,a}}$, and $g_{0}(k,a)=\bE[\sum\limits^{n^{0,a}}_{j=k}{n^{0,a}\choose j}(q_{Y,a}(\alpha))^{j}(1-q_{Y,a}(\alpha))^{n^{0,a}-j}]$ with $Q^{0,a}\sim Beta(k+1,n^{0,a}-k)$. Then we have:
\begin{align*}
    \bP(|DPE(\phi)|>\alpha) \leq g_{1}(k^{1,1},1) + g_{1}(k^{1,0},0)+g_{0}(k^{0,1},1) + g_{0}(k^{0,0},0). 
\end{align*}

\end{proposition}

\begin{assumption}
$n^{0,0} \geq \lceil \frac{\log \frac{\delta}{4}} {\log (1-\frac{\alpha}{1-p_{Y,0}})}\rceil$, $n^{0,1} \geq \lceil \frac{\log \frac{\delta}{4}} {\log (1-\frac{\alpha}{1-p_{Y,1}})}\rceil$, $n^{1,0} \geq \lceil \frac{\log \frac{\delta}{4}} {\log (\frac{p_{Y,1} - \alpha}{p_{Y,0}})}\rceil$, $n^{1,1} \geq \lceil \frac{\log \frac{\delta}{4}} {\log (\frac{p_{Y,0} - \alpha}{p_{Y,1}})}\rceil\text {, in which }\lceil\cdot\rceil \text { denotes the ceiling function. }$
\label{a6}
\end{assumption} 

\begin{theorem}
Under Assumption \ref{a6}, we have $|DEA(\hat\phi_i)| \leq \alpha$ with probability $1-\delta$, for each $i\in\{1,\ldots,M\}$.
\end{theorem}

\begin{corollary}
Under Assumption \ref{a6}, we have $|DEA(\hat\phi)| \leq \alpha$ with probability $(1-\delta)^{M}$, where $M$ is the size of the candidate set.
\end{corollary}

\subsection{Implementation details and additional experiments}
\label{AA9}
From Lemma \ref{eo}, we adopt a new way of building a much smaller candidate set. Note that our shrunk candidate set for Equality of Opportunity is:
$$
\begin{aligned}
K^{\prime} =  &\{(k^{1,0}, u_1(k^{1,0})) \mid L_1(k^{1,0}, u_1(k^{1,0})) \leq \delta\}.
\end{aligned}
$$

Since Equalized Odds constraint is an extension of Equality of Opportunity, our target classifier should be in $K^{\prime}$.

To select our target classifier, it's sufficient to add a condition of similar false positive rate between privileged and unprivileged groups. Specifically, we choose our final candidate set as below:

$$
\begin{aligned}
K^{\prime\prime} =  &\{(k^{1,0}, u_1(k^{1,0})) \mid L_1(k^{1,0}, u_1(k^{1,0})) \leq \delta, L_0(k^{0,0}, k^{1,0}) \leq \delta\}. 
\end{aligned}
$$
We also did experiments on other benchmark datasets.

First, we apply FaiREE to German Credit dataset \cite{kamiran2009classifying}, whose task is to predict whether a bank account holder's credit is \textit{good} or \textit{bad}. The protected attribute is gender, and the sample size is 1000, with 800 training samples and 200 test samples. To facilitate the numerical study, we randomly split data into training set, calibration set, and test set at each repetition and repeat 500 times. 
\begin{table}[htbp]
\small
\caption{Result of different methods on German Credit dataset}\label{tab:german}
\begin{center}
\resizebox{\columnwidth}{!}{\setlength{\tabcolsep}{1.5mm}{\begin{tabular}{c|c|c|c|c|c|c|c|c|c|c|c|c}
\toprule
{} & {Eq} & {C-Eq} & {ROC} & \multicolumn{3}{|c|}{FairBayes} & \multicolumn{3}{|c|}{FaiREE-EOO} & \multicolumn{3}{|c}{FaiREE-EO}\\
\midrule
{$\alpha$} & {/} &{/} &{/} & {0.07} & {0.1}& {0.14}& {0.07}& {0.1}& {0.14}& {0.07}& {0.1}& {0.14}\\
\midrule
{$\overline{|DEOO|}$} & {0.078} & {0.093}& {0.100}& {0.126}& {0.125}& {0.126}& {0.020}& {0.029}& {0.034}& {0.001}& {0.025}& {0.048}\\
{$|DEOO|_{95}$} & {0.179} & {0.127}& {0.267}& {0.160}& {0.182}& {0.186}& {0.066}& {0.097}& {0.130}& {0.004}& {0.084}& {0.120}\\
{$\overline{|DPE|}$} & {0.109} & {0.072}& {0.138}& {/}& {/}& {/}& {/}& {/}& {/}& {0.041}& {0.063}& {0.092}\\
{${|DPE|}_{95}$} & {0.235} & {0.114}& {0.334}& {/}& {/}& {/}& {/}& {/}& {/}& {0.059}& {0.097}& {0.133}\\
\midrule
{$\overline{ACC}$} & {0.707} & {0.720}& {0.591}& {0.722}& {0.723}& {0.723}& {0.717}& {0.729}& {0.745}& {0.702}& {0.721}& {0.722}\\
\midrule
\end{tabular}}}
\end{center}
\end{table}

Then, we apply FaiREE to Compas Score dataset \cite{angwin2016machine}, whose task is to predict whether a person will conduct crime in the future. The protected attribute is gender, and the sample size is 5278, with 4222 training samples and 1056 test samples. To facilitate the numerical study, we randomly split data into training set, calibration set and test set at each repetition and repeat for 500 times. 
\begin{table}[htbp]
\small
\caption{Result of different methods on Compas Score dataset}\label{tab:compas}
\begin{center}
\resizebox{\columnwidth}{!}{\setlength{\tabcolsep}{1.5mm}{\begin{tabular}{c|c|c|c|c|c|c|c|c|c|c|c|c}
\toprule
{} & {Eq} & {C-Eq} & {ROC} & \multicolumn{3}{|c|}{FairBayes} & \multicolumn{3}{|c|}{FaiREE-EOO} & \multicolumn{3}{|c}{FaiREE-EO}\\
\midrule
{$\alpha$} & {/} &{/} &{/} & {0.07} & {0.1}& {0.14}& {0.07}& {0.1}& {0.14}& {0.07}& {0.1}& {0.14}\\
\midrule
{$\overline{|DEOO|}$} & {0.083} & {0.642}& {0.070}& {0.077}& {0.109}& {0.136}& {0.026}& {0.033}& {0.042}& {0.027}& {0.049}& {0.098}\\
{$|DEOO|_{95}$} & {0.101} & {0.684}& {0.174}& {0.145}& {0.186}& {0.250}& {0.066}& {0.097}& {0.131}& {0.068}& {0.090}& {0.140}\\
{$\overline{|DPE|}$} & {0.025} & {0.291}& {0.067}& {/}& {/}& {/}& {/}& {/}& {/}& {0.026}& {0.048}& {0.060}\\
{${|DPE|}_{95}$} & {0.047} & {0.332}& {0.148}& {/}& {/}& {/}& {/}& {/}& {/}& {0.057}& {0.092}& {0.114}\\
\midrule
{$\overline{ACC}$} & {0.629} & {0.664}& {0.652}& {0.658}& {0.658}& {0.659}& {0.654}& {0.660}& {0.672}& {0.623}& {0.654}& {0.669}\\
\midrule
\end{tabular}}}
\end{center}
\end{table}

We further generate a synthetic model where trained classifiers are more informative.
\begin{enumerate}
    \item[\textbf{Model 3.}] 
    We generate the protected attribute $A$ and label $Y$ with the probability, location parameter and scale parameter the same as Model 1. The dimension of features is set to 60, and we generate features with $x^{0,0}_{i,j} \stackrel{i.i.d.}{\sim} t(1)$, $x^{0,1}_{i,j} \stackrel{i.i.d.}{\sim} t(4)$, $x^{1,0}_{i,j} \stackrel{i.i.d.}{\sim} \chi^2_1$ and $x^{1,1}_{i,j} \stackrel{i.i.d.}{\sim} \chi^2_4$, for $j=1,2,...,60$.
\end{enumerate} 
\begin{table}[htbp]
\small
\caption{Experimental studies under Model 3. Here $\overline{|DEOO|}$ denotes the sample average of the absolute value of $DEOO$ defined in Eq.~(\ref{deoo}), and $|DEOO|_{95}$ denotes the sample upper 95\% quantile. $\overline{|DPE|}$ and ${|DPE|}_{95}$ are defined similarly for $DPE$ defined in Eq.~(\ref{dpe}). $\overline{ACC}$ is the sample average of accuracy. We use ``/" in the $DPE$ line because FairBayes and FaiREE-EOO are not designed to control $DPE$.}
\begin{center}
\resizebox{\columnwidth}{!}{\setlength{\tabcolsep}{1.5mm}{\begin{tabular}{c|c|c|c|c|c|c|c|c|c|c|c|c}
\toprule
{} & {Eq} & {C-Eq} & {ROC} & \multicolumn{3}{|c|}{FairBayes} & \multicolumn{3}{|c|}{FaiREE-EOO}& \multicolumn{3}{|c}{FaiREE-EO}\\
\midrule
{$\alpha$} & {/} &{/} &{/} & {0.04} & {0.06}& {0.08}& {0.04}& {0.06}& {0.08}& {0.04}& {0.06}& {0.08}\\
\midrule
{$\overline{|DEOO|}$} & {0.041} & {0.020}& {0.034}& {0.048}& {0.062}& {0.076}& {0.018}& {0.025}& {0.033}& {0.015}& {0.032}& {0.034}\\
{$|DEOO|_{95}$} & {0.115} & {0.042}& {0.070}& {0.080}& {0.113}& {0.143}& {0.038}& {0.054}& {0.076}& {0.036}& {0.058}& {0.069}\\
{$\overline{|DPE|}$} & {0.093} & {0.106}& {0.062}& {/}& {/}& {/}& {/}& {/}& {/}& {0.026}& {0.034}& {0.046}\\
{${|DPE|}_{95}$} & {0.272} & {0.191}& {0.101}& {/}& {/}& {/}& {/}& {/}& {/}& {0.039}& {0.055}& {0.077}\\
\midrule
{$\overline{ACC}$} & {0.887} & {0.921}& {0.834}& {0.898}& {0.901}& {0.912}& {0.946}& {0.950}& {0.963}& {0.900}& {0.914}& {0.933}\\
\midrule
\end{tabular}}}
\end{center}
\label{table7}
\end{table}

\subsubsection{Comparison with more algorithms}
In this subsection, we further compare FaiREE with more baseline algorithms that are designed for achieving Equalized Odds or Equality of Opportunity, including pre-processing algorithms (Fairdecision in \cite{kilbertus2020fair} and LAFTR in \cite{madras2018learning}) and in-processing algorithms (Meta-cl in \cite{celis2019classification} and  Adv-debias in \cite{zhang2018mitigating}) under synthetic settings Model 1 and Model 2 described in Section~\ref{sec:5.1}, and the real dataset Adult Census in Section~\ref{sec:5.2}. The results are summarized in Tables~ \ref{table10}, \ref{table8}, and \ref{table9}.

From the experimental results, we can find that FaiREE has favorable results over these baseline methods, with respect to fairness and accuracy. In particular, the experimental results indicate that while the baseline methods are designed to minimize the fairness violation as much as possible (i.e. set $\alpha=0$), these methods are unable to have an exact control of the fairness violation to a desired level $\alpha$. For example, in the analysis of Adult Census dataset, the 95\% quantile of the DEOO fairness violations of Fairdecision is 0.078, and that of LAFTR, Meta-cl and Adv-debias are all above 0.2. Moreover, our results found that If we allow the same fairness violation of DEOO and DEP for our proposed method FaiREE, we have a much higher accuracy (0.845) compared to the accuracy of those four baseline methods.


\begin{table}[H]
\small
\caption{Results of different methods on Adult Census dataset. Here $\overline{|DEOO|}$ denotes the sample average of the absolute value of $DEOO$ defined in Eq.~(\ref{deoo}), and $|DEOO|_{95}$ denotes the sample upper 95\% quantile. $\overline{|DPE|}$ and ${|DPE|}_{95}$ are defined similarly for $DPE$ defined in Eq.~(\ref{dpe}). $\overline{ACC}$ is the sample average of accuracy. We use ``/" in the $DPE$ line because Fairdecision and FaiREE-EOO are not designed to control $DPE$.}\label{tab:morecom3}
\begin{center}
\resizebox{\columnwidth}{!}{\setlength{\tabcolsep}{1.5mm}{\begin{tabular}{c|c|c|c|c|c|c|c|c|c|c}
\toprule
{} & {Fairdecision} & {LAFTR} & {Meta-cl} & { Adv-debias} & \multicolumn{3}{|c|}{FaiREE-EOO}& \multicolumn{3}{|c}{FaiREE-EO}\\
\midrule
{$\alpha$} & {/} &{/} &{/} & {/} & {0.07}& {0.1}& {0.14}& {0.07}& {0.1}& {0.14}\\
\midrule
{$\overline{|DEOO|}$} & {0.041} & {0.124}& {0.172}& {0.199}& {0.034}& {0.039}& {0.066}& {0.002}& {0.039}& {0.067}\\
{$|DEOO|_{95}$} & {0.078} & {0.203}& {0.253}& {0.248}& {0.065}& {0.090}& {0.124}& {0.008}& {0.094}& {0.125}\\
{$\overline{|DPE|}$} & {/} & {0.044}& {0.194}& {0.074}& {/}& {/}& {/}& {0.030}& {0.066}& {0.074}\\
{${|DPE|}_{95}$} & {/} & {0.083}& {0.271}& {0.094}& {/}& {/}& {/}& {0.056}& {0.078}& {0.086}\\
\midrule
{$\overline{ACC}$} & {0.772} & {0.822}& {0.688}& {0.791}& {0.845}& {0.846}& {0.847}& {0.512}& {0.845}& {0.846}\\
\midrule
\end{tabular}}}
\end{center}
\label{table10}
\end{table}

\begin{table}[H]
\small
\caption{Experimental studies under Model 1, with the same notation as Table \ref{table10}} 
\begin{center}
\resizebox{\columnwidth}{!}{\setlength{\tabcolsep}{1.5mm}{\begin{tabular}{c|c|c|c|c|c|c|c|c|c|c}
\toprule
{} & {Fairdecision} & {LAFTR} & {Meta-cl} & { Adv-debias} & \multicolumn{3}{|c|}{FaiREE-EOO}& \multicolumn{3}{|c}{FaiREE-EO}\\
\midrule
{$\alpha$} & {/} &{/} &{/} & {/} & {0.08}& {0.12}& {0.16}& {0.08}& {0.12}& {0.16}\\
\midrule
{$\overline{|DEOO|}$} & {0.072} & {0.081}& {0.028}& {0.062}& {0.028}& {0.046}& {0.063}& {0.025}& {0.031}& {0.042}\\
{$|DEOO|_{95}$} & {0.177} & {0.145}& {0.108}& {0.226}& {0.073}& {0.115}& {0.157}& {0.079}& {0.108}& {0.133}\\
{$\overline{|DPE|}$} & {/} & {0.061}& {0.118}& {0.179}& {/}& {/}& {/}& {0.039}& {0.042}& {0.045}\\
{${|DPE|}_{95}$} & {/} & {0.104}& {0.272}& {0.412}& {/}& {/}& {/}& {0.075}& {0.084}& {0.106}\\
\midrule
{$\overline{ACC}$} & {0.616} & {0.533}& {0.620}& {0.645}& {0.621}& {0.657}& {0.669}& {0.552}& {0.562}& {0.615}\\
\midrule
\end{tabular}}}
\end{center}
\label{table8}
\end{table}

\begin{table}[H]
\small
\caption{Experimental studies under Model 2, with the same notation as Table \ref{table10}.}\label{tab:morecom2}
\begin{center}
\resizebox{\columnwidth}{!}{\setlength{\tabcolsep}{1.5mm}{\begin{tabular}{c|c|c|c|c|c|c|c|c|c|c}
\toprule
{} & {Fairdecision} & {LAFTR} & {Meta-cl} & { Adv-debias} & \multicolumn{3}{|c|}{FaiREE-EOO}& \multicolumn{3}{|c}{FaiREE-EO}\\
\midrule
{$\alpha$} & {/} &{/} &{/} & {/} & {0.08}& {0.12}& {0.16}& {0.08}& {0.12}& {0.16}\\
\midrule
{$\overline{|DEOO|}$} & {0.675} & {0.450}& {0.094}& {0.096}& {0.027}& {0.047}& {0.073}& {0.028}& {0.035}& {0.047}\\
{$|DEOO|_{95}$} & {0.744} & {0.633}& {0.208}& {0.263}& {0.075}& {0.112}& {0.153}& {0.077}& {0.114}& {0.143}\\
{$\overline{|DPE|}$} & {/} & {0.502}& {0.120}& {0.140}& {/}& {/}& {/}& {0.041}& {0.044}& {0.056}\\
{${|DPE|}_{95}$} & {/} & {0.686}& {0.312}& {0.418}& {/}& {/}& {/}& {0.071}& {0.090}& {0.127}\\
\midrule
{$\overline{ACC}$} & {0.584} & {0.647}& {0.606}& {0.628}& {0.595}& {0.627}& {0.639}& {0.575}& {0.589}& {0.606}\\
\midrule
\end{tabular}}}
\end{center}
\label{table9}
\end{table}

\subsection{Supplementary figures}
\begin{figure}[htbp]
    \centering
    \includegraphics[width=14cm]{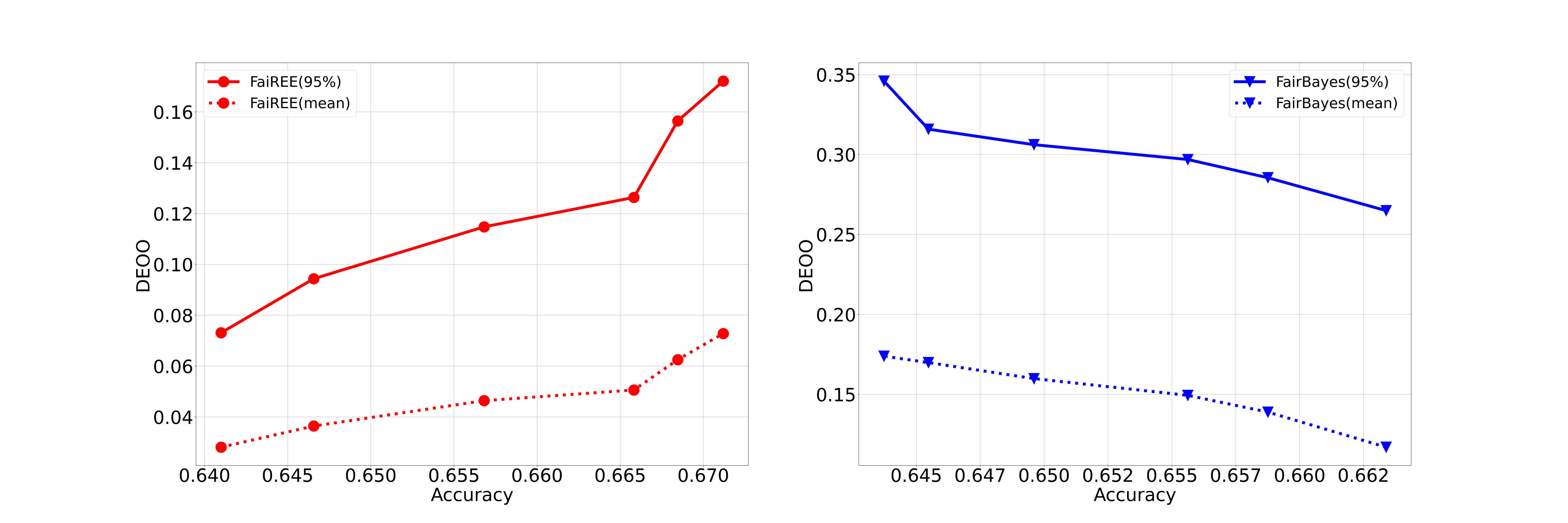}
    \caption{DEOO v.s. Accuracy, as a complementary figure for Figure \ref{compare}}
    \label{fig:my_label}
\end{figure}

\begin{figure}[htbp]
    \centering
    \includegraphics[scale = 0.5]{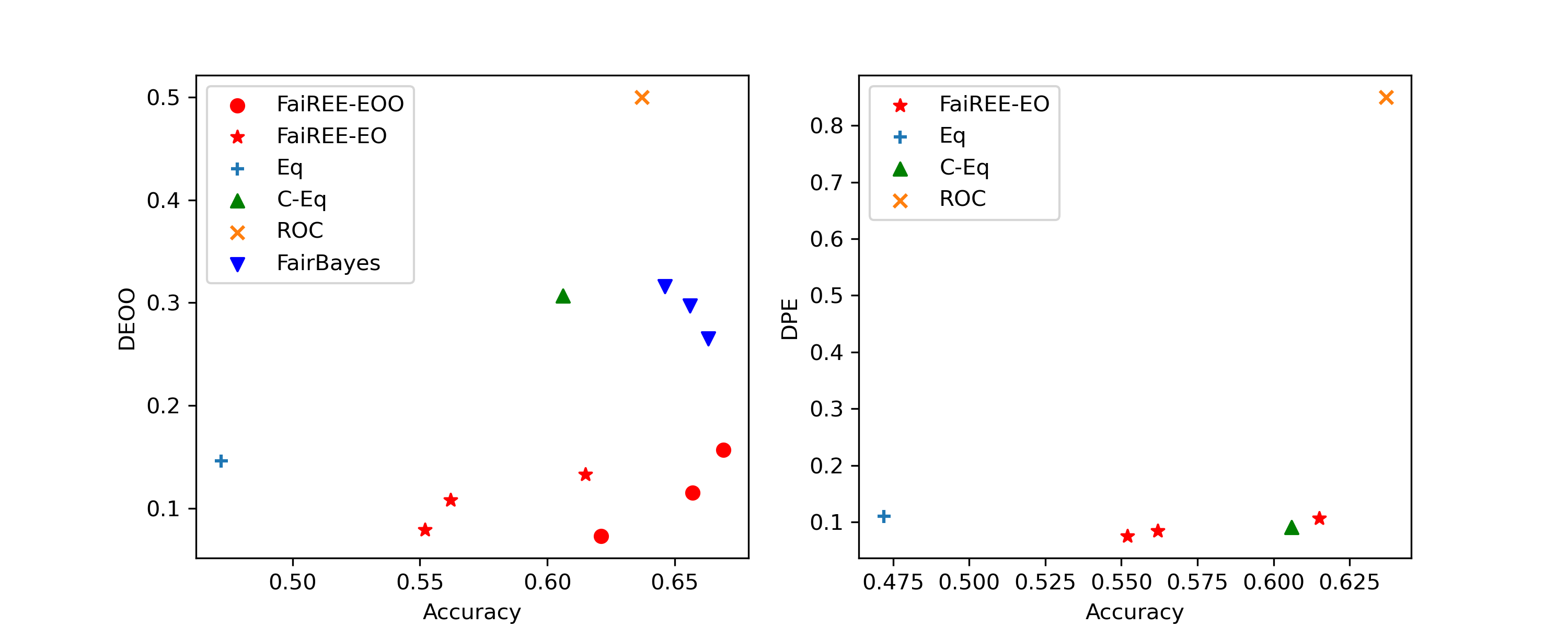}
    \caption{DEOO v.s. Accuracy \& DPE v.s. Accuracy for Model 1}
    \label{fig:my_label}
\end{figure}

\begin{figure}[htbp]
    \centering
    \includegraphics[scale = 0.5]{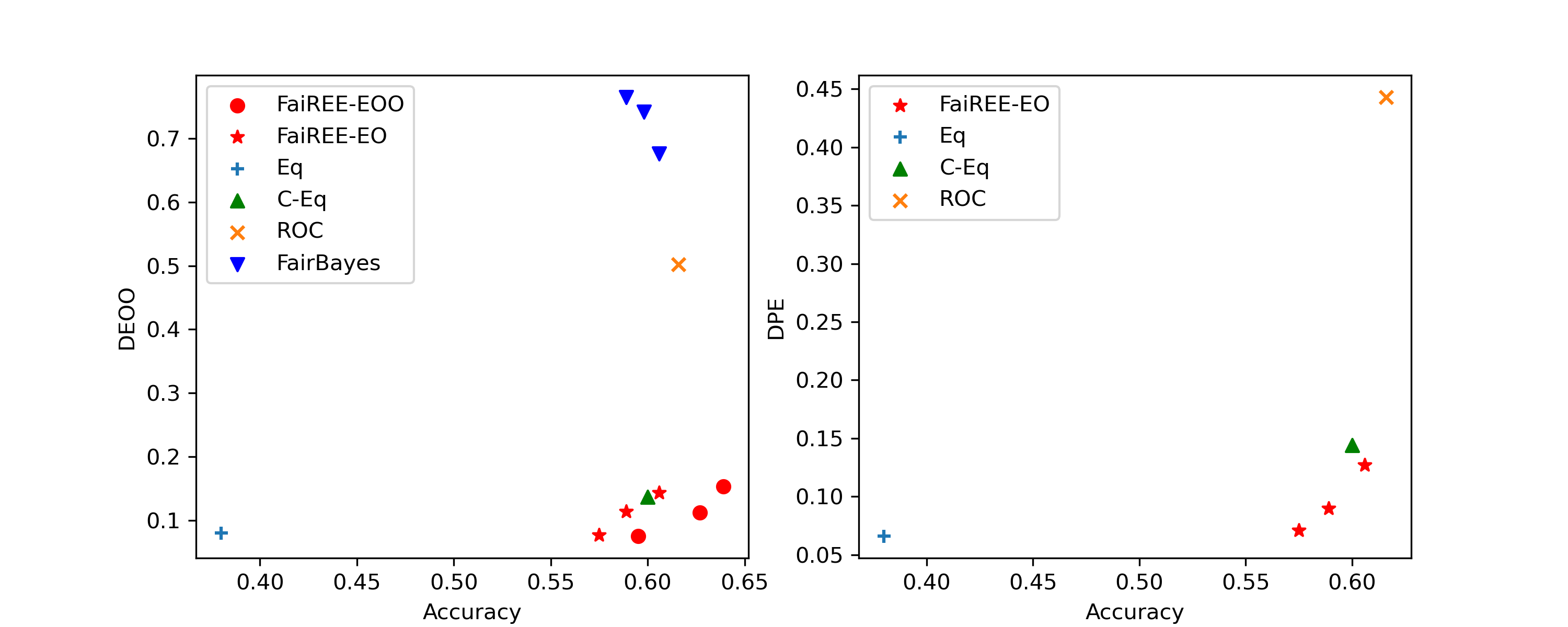}
    \caption{DEOO v.s. Accuracy \& DPE v.s. Accuracy for Model 2}
    \label{fig:my_label}
\end{figure}

\begin{figure}[htbp]
    \centering
    \includegraphics[scale = 0.5]{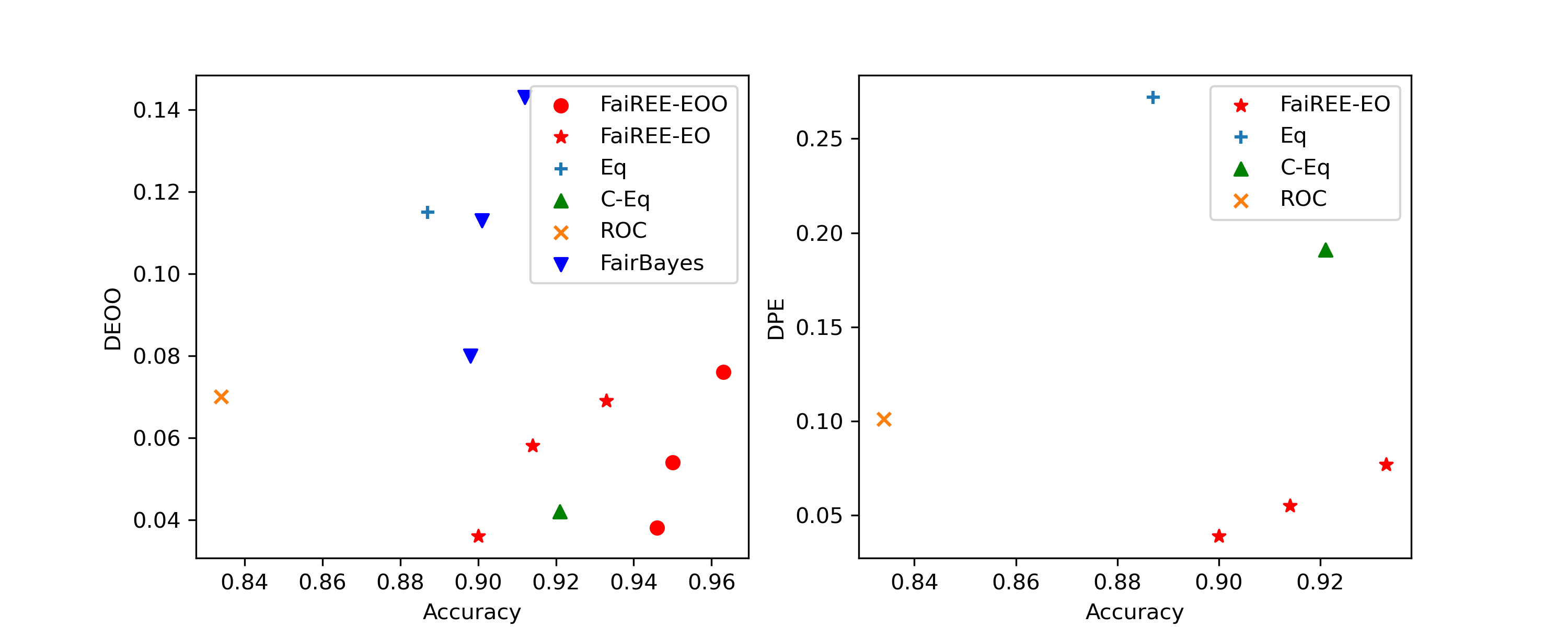}
    \caption{DEOO v.s. Accuracy \& DPE v.s. Accuracy for Model 3}
    \label{fig:my_label}
\end{figure}

\begin{figure}[htbp]
    \centering
    \includegraphics[scale = 0.5]{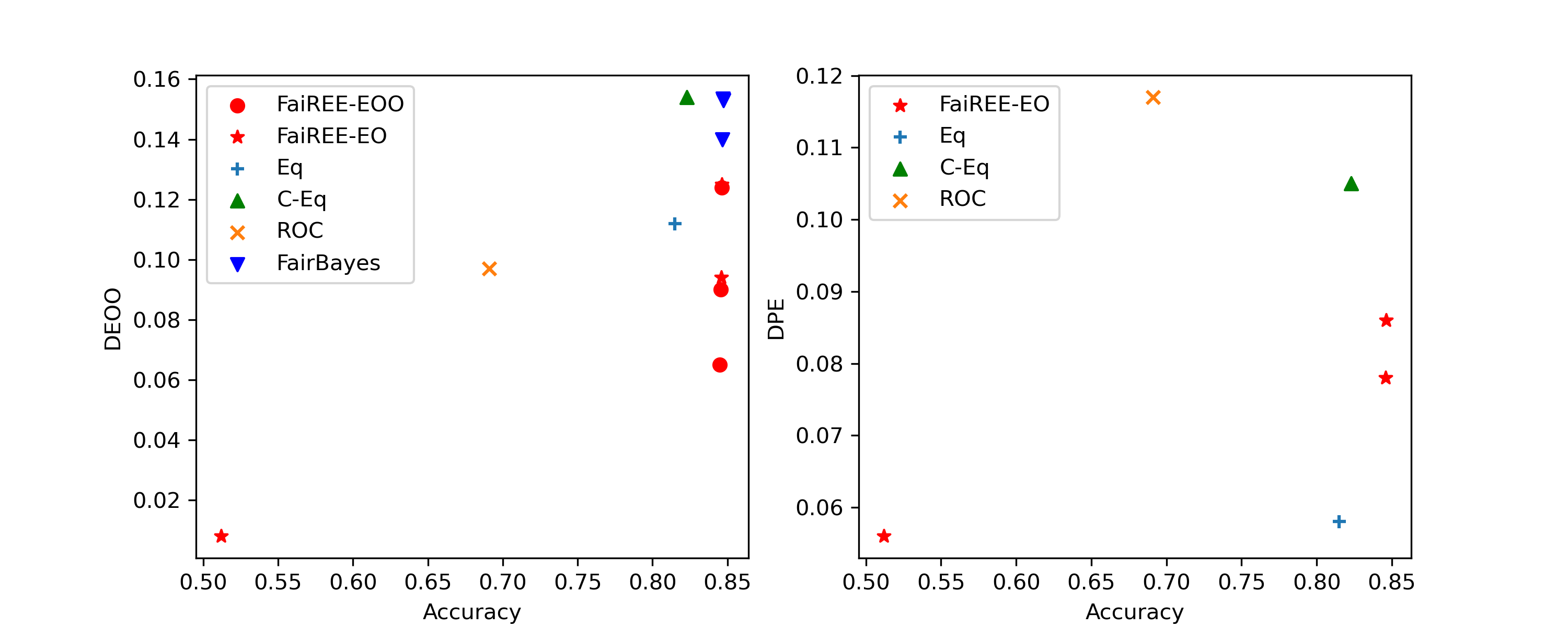}
    \caption{DEOO v.s. Accuracy \& DPE v.s. Accuracy for Adult Census dataset}
    \label{fig:my_label}
\end{figure}

\begin{figure}[htbp]
    \centering
    \includegraphics[scale = 0.5]{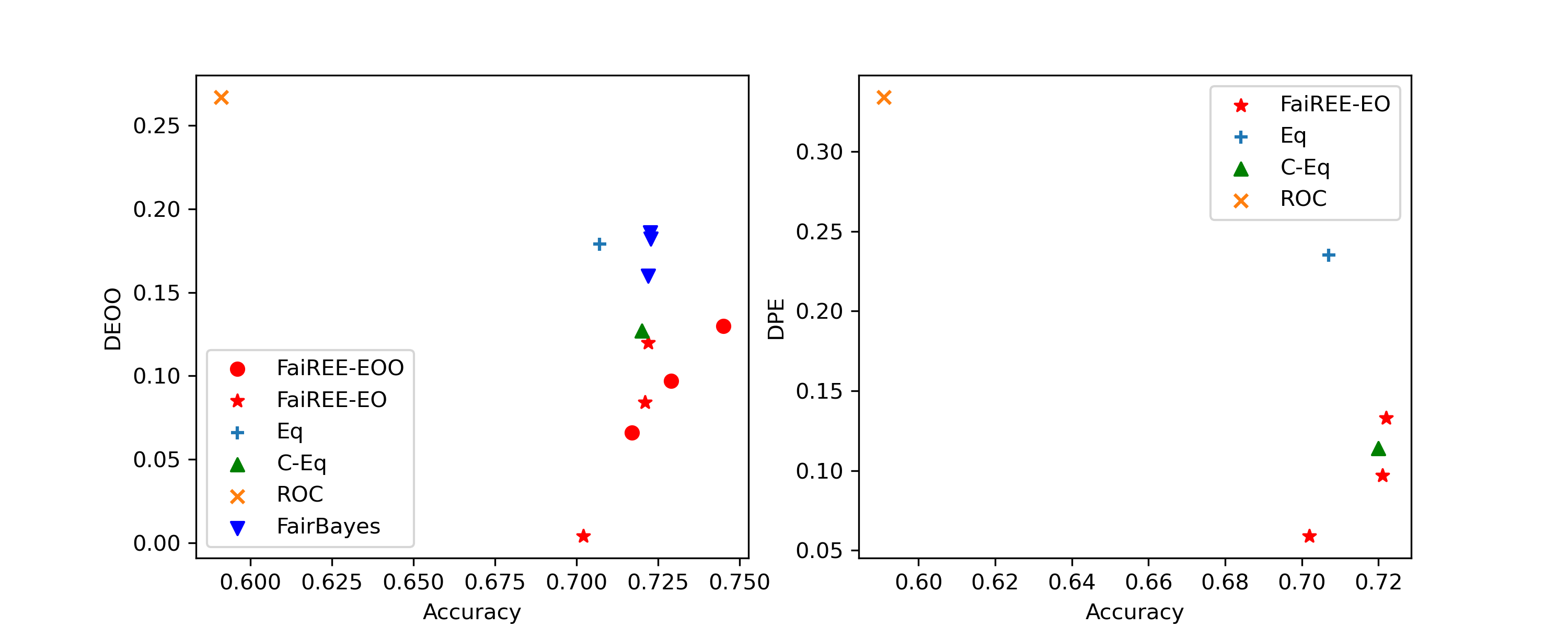}
    \caption{DEOO v.s. Accuracy \& DPE v.s. Accuracy for German Credit dataset}
    \label{fig:my_label}
\end{figure}

\begin{figure}[htbp]
    \centering
    \includegraphics[scale = 0.5]{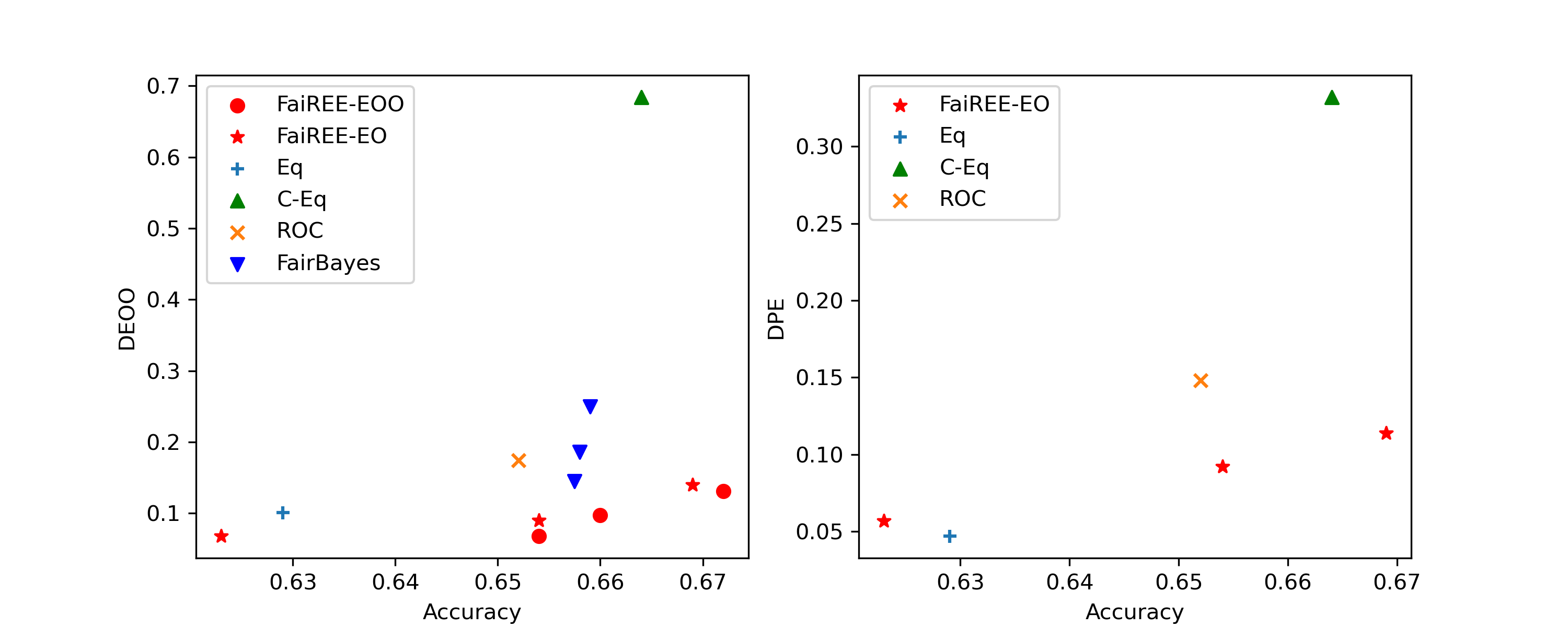}
    \caption{DEOO v.s. Accuracy \& DPE v.s. Accuracy for Compas Score dataset}
    \label{fig:my_label}
\end{figure}
\end{document}